\documentclass[12pt]{article}

\usepackage{amsmath,amsthm}
\usepackage{amssymb}
\usepackage{amsfonts}
\usepackage{amscd}
\usepackage{dsfont} 
\usepackage{mathtools} 
\usepackage{braket} 
\usepackage{color} 

\input{xy}
\xyoption{all}
\newtheorem{thm}{Theorem}[section]
\newtheorem{theorem}[thm]{Theorem}

\newtheorem{proposition}[thm]{Proposition}

\theoremstyle{definition}
\newtheorem{definition}[thm]{Definition}
\newtheorem{example}[thm]{Example}

\newtheorem{remark}[thm]{Remark}

\newtheorem{observation}[thm]{Observation}

\usepackage[english]{babel}
\usepackage{newlfont}
\usepackage{color}
\usepackage{hyperref}
\usepackage{multirow}
\usepackage{wrapfig}
\usepackage{cleveref}
\usepackage{caption}
\usepackage{float}
\usepackage{refcount}
\usepackage{gettitlestring}
\usepackage{csquotes}
\usepackage[dvipsnames, table]{xcolor}
\usepackage{dsfont}

\newcommand{\cN}{\mathcal{N}}
\newcommand{\R}{\mathbf{R}}
\newcommand{\N}{\mathbf{N}}
\newcommand{\bE}{\mathbf{E}}

\newcommand{\M}{\mathrm{M}}

\newcommand{\beq}{\begin{equation}}
\newcommand{\eeq}{\end{equation}}
\newcommand{\lra}{\longrightarrow}
\newcommand{\rank}{\mathrm{rank}}
\newcommand{\Span}{\mathrm{span}}
\newcommand{\RELU}{\mathrm{RELU}}
\newcommand{\KL}{\mathrm{KL}}

\begin{document}

\centerline{\Large
\bf Deep Learning and Geometric Deep Learning:}

\medskip 
\centerline{\Large
\bf an introduction for mathematicians and physicists} 

\medskip
\centerline{R. Fioresi, F. Zanchetta\footnote{This research was supported by Gnsaga-Indam, by
COST Action CaLISTA CA21109 and by
HORIZON-MSCA-2022-SE-01-01 CaLIGOLA.}}

\centerline{FaBiT, via San Donato 15, 41127 Bologna, Italy}

\centerline{rita.fioresi@unibo.it, ferdinando.zanchett2@unibo.it}

\begin{abstract}
    In this expository paper we want to give a brief introduction, with few key references for further reading, to the inner functioning of the new and successfull algorithms of Deep Learning and Geometric Deep Learning with a focus on Graph Neural Networks. We go over the key ingredients for  these algorithms: the score and loss function and we explain the main steps for the training of a model. We do not aim to give a complete and exhaustive treatment, but we isolate few concepts to give a fast introduction
    to the subject. We provide some appendices 
    to complement our treatment discussing Kullback-Leibler divergence, regression, Multi-layer Perceptrons and the Universal Approximation Theorem. 
\end{abstract}

\tableofcontents
\section{Introduction}\label{intro-sec}
The recent revolution in machine learning, including
the spectacular success of the
Deep Learning (DL) algorithms
\cite{bronstein2016, bronstein2021, lecun2015}, 
challenges mathematicians and physicists
to provide 
models that explain the elusive mechanisms
inside them. Indeed, the popularity of
Deep Learning, a class of machine learning algorithms that aims to solve problems by extracting high-level features from some raw input, is rising as it is being employed successfully to solve difficult practical problems in many different fields such as speech recognition \cite{SRNH}, natural language processing \cite{NLPGOOGLE, hinton2012deep}, image recognition \cite{imagenet}, drug discovery \cite{AtomNet}, bioinformatics \cite{DeepBind} and medical image analysis \cite{DLMedIMG} just to cite a few, but the list 
is much longer. Recently at Cern the power of deep learning was employed for the analysis of LHC 
(Large Hadron Collider) data \cite{Guest2018DeepLA}; particle physics is more and more being investigated by
these new and powerful methods (see the review \cite{Feickert2021ALR} and refs. therein).
To give a concrete examples of more practical applications, Convolutional
Neural Networks (CNNs, \cite{LeCunzip}, \cite{lecun2015}), particular DL algorithms, were employed in the ImageNet challenge
\cite{kr-imagenet, imagenet},
a supervised classification task challenge, where participants proposed
their algorithms to classify a vast dataset of
images belonging to 1000 categories. The CNNs algorithms proposed by
several researchers including LeCun, Hinton et al. \cite{lecun2015, 
kr-imagenet} surpassed the
human performance and contributed to establish a new paradigm in machine
learning.

\medskip
The purpose of this paper is to elucidate some of the main mechanisms
of the algorithms of Deep Learning and of 
Geometric Deep Learning \cite{bronstein2016,bronstein2021} (GDL), an adaptation of Deep Learning for data organized
in a graph structure
\cite{kw}. To this end, we shall focus on two very important types of algorithms: CNNs and Graph Neural Networks (GNNs, \cite{GRLHamilton}) for  supervised classification tasks.
Convolutional neural networks, DL algorithms already part of the family of GDL algorithms, are extremely successful
for problems involving data with a {\sl grid structure}, in other
words a regularly organized structure, where the
mathematical operation of convolution,
in its discrete version, makes sense. On the other hand, when we
deal with datasets having an underlying geometric graph structures, it is necessary to adapt the notion of convolution,
since every node has, in principle, a different neighbourhood and local topology. This motivates the introduction of Graph Neural Networks.

We do not plan to give a complete treatment of all the types of CNNs and GNNs, but
we wish to give a rigorous introduction for the mathematicians and
physicists that wish to know more about both the implementation and the theory behind these algorithms.

\medskip
The organization of our paper is as follows.

In Sec. \ref{sc-sec}, we describe CNNs for supervised classification,
with focus on the image classification task, since historically this
is one of the main applications that helped to establish the Deep Learning popularity among the
classification algorithms. In this section we describe
the various steps in the creation of a {\sl model}, that is
the process, called {\sl training} in which we determine the
parameters of a network, which perform best for the given classification task.
This part is common to many algorithms, besides Deep Learning ones.

In Sec. \ref{dl-sec}, we focus on the {\sl score} and the {\sl loss}
functions, 
describing some examples of such functions, peculiar to Deep Learning, that are used
routinely. 
Given a datum, for example an image, the score function assigns
to it, the score corresponding to each class. Hence it will allow us
to classify the datum, by assigning to it the class with
the highest score. The loss of a datum measures the error committed during
the classification: the higher the loss, the further we are from
the correct class of the datum. 
Some authors say that the purpose of
the training is to ``minimize'' the loss, and we shall give a precise meaning to this intuitive statement.  

In Sec. \ref{gdl-sec}, we introduce some graph neural networks for supervised node classification tasks.
After a brief introduction to graphs and their laplacians, we start our description of the score function in this context. In fact, it is necessary to modify the ``convolutional
layers'', that is the convolutional functions appearing in the expression
of the score function, to adapt them to the data in graph form, which,
in general, is not grid-like. This leads to the so called ``message passing''
mechanism, the heart of the graph neural networks algorithms
and explained in Subsec. \ref{heat-sec} describing the
convolution on graphs and its intriguing relation with
the heat equation.  

In the end we provide few appendices to deepen the mathematical
treatment that we have not included in the text not to disrupt the reading.

\section{Supervised Classification}\label{sc-sec}

Deep Learning is a very successful family of machine learning algorithms. We can use it to solve both supervised and unsupervised problems for both regression or classification tasks. In this paper, we shall focus on supervised classification tasks: as supervised regression tasks are handled similarly, we will mention them briefly in Appendix \ref{regr-app}.
To solve a classification problem in the supervised learning setting, we have data points belonging to the euclidean
space $\R^d$ and each point has a label, tipically an integer number
between $0$ and $C-1$, where $C$ is the number
of {\it classes}. The goal is to ``learn'', that is to determine, a function that associates
data points with the correct label $F:\R^d \lra \{0,\dots, C-1\}$,
through the process of \textit{training} that we discuss below.

\medskip
We shall focus our attention on a classification task the algorithm of
Deep Learning is particularly effective with: 
{\sl supervised image classification}. 

\subsection{Supervised image classification  datasets}\label{scd-sec}

In this supervised task,
we have a database of images, each coming with
a {\it ground truth} label, that is a label representing 
the class of the given image. Hence, for supervised classification,
it is necessary to have a {\sl labelled
dataset}; producing such databases is time consuming in concrete applications, where
labelling may take a lot of human effort. However, 
there are two very important examples of
such datasets representing
a key reference for all researchers in this field:
the \textit{MNIST} and {\it CIFAR} datasets \cite{kr-cifar, lecun2010mnist}. 
Datasets of this kind are called ``benchmark'' datasets, they are free and publicly
available and they can be easily downloaded from several
websites, including
the colab platform\footnote{colab.research.google.com}. for example.
The images in these datasets are
already labelled, hence they are ready for supervised image classification.
They are perfect to start ``hands on'' experiments\footnote{See the tutorials 
publicly available on the colab platform.}.

\medskip
Let us briefly describe these datasets, since they represent the
first and most accessible way to learn how to construct
and run successfully a deep
learning algorithm.

\medskip
The MNIST dataset
contains $70000$ black and white images of handwritten digits 
from 0 to 9 (see Fig. \ref{fig:mnist}). 
\vskip 0.05in
\begin{figure}[ht]
\begin{center}
\centerline{\includegraphics[width=3in]{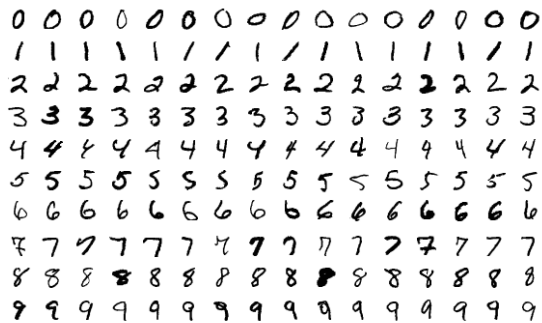}}
\vskip 0.05in
\caption{Examples from MNIST dataset}
\label{fig:mnist}
\end{center}
\vskip -0.2in
\end{figure}
Each image consists of
$28 \times 28$ pixels, hence it is represented by a matrix
of size $28 \times 28$, with integral entries between $0$ and
$255$ representing the grayscale of the pixel. Effectively,
an image is a point in $\R^d$, for $d=28 \times 28=784$.
\vskip 0.05in
\begin{figure}[ht]
\begin{center}
\centerline{\includegraphics[width=5.5in]{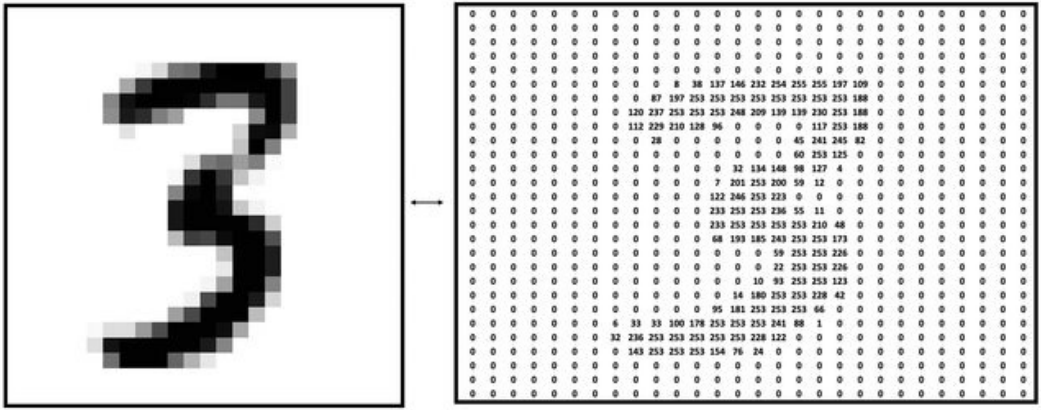}}
\vskip 0.05in
\caption{Matrix values for a sample of number 3 in MNIST, \cite{boroojerdi}}
\label{fig:mnist1}
\end{center}
\vskip -0.2in
\end{figure}
The MNIST dataset
comes already
divided into a training set of 60000 examples and a test set of 10000
examples. We shall see the meaning of such partition in the sequel.

\medskip
Another important benchmark dataset is the CIFAR10 database, containing 60000
natural color images divided into 10 classes, ranging from airplanes, cars etc.
to cat, dogs, horses etc.
Each image consists of $32\times 32$ pixels; each pixel has three numbers
associated with it, between 0 and 255 corresponding
to one of the RGB channels\footnote{RGB means
Red, Green and Blue the combination
of shades of these three colors gives a color image.}.
So, practically, every image is an
element in $\R^d$, with $d=32 \times 32 \times 3=3072$.
\vskip 0.05in
\begin{figure}[ht]
\begin{center}
\centerline{\includegraphics[width=4.5in]{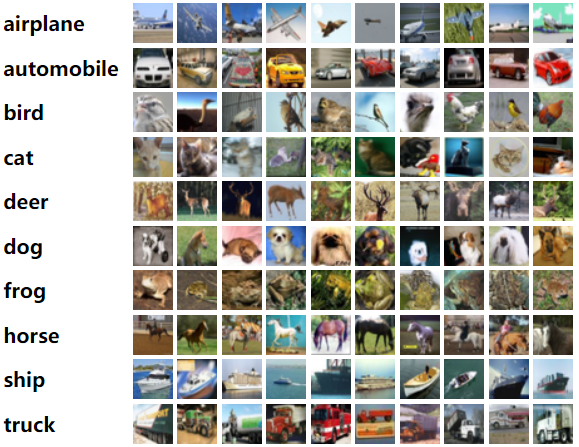}}
\vskip 0.05in
\caption{Examples from CIFAR10 dataset (www.cs.toronto.edu)}
\label{fig:mnist}
\end{center}
\vskip -0.2in
\end{figure}

In this dataset each datum has a larger dimension than in MNIST,
where $d=28 \times 28$. This larger dimension, coupled with a greater
variability among samples, makes 
the training longer and more expensive; we shall see more on this later on.

\medskip
In supervised image classification tasks, 
the dataset is typically divided into three disjoint subsets\footnote{
  The ``cross validation'' technique uses a part of
the training set as validation set, but, for the sake of
clarity, we shall not describe this common
practice here.}: {\sl training}, {\sl validation} and {\sl test} sets.

\medskip
1. {\it Training set}: it is used to determine
  the function that associates a label to a given image. Such function
  depends on parameters $w \in \R^p$, called \textit{weights}.
  The process during
  which we determine the weights, by successive
  iterations, is called \textit{training}; the weights
  are initialized randomly (for some standard random initialization algorithms, see \cite{Xavier})
  at the beginning of the training.
  We refer to a set of weights, that we have determined through
  a training,
  as a \textit{model}. Notice that, since the weights are initialized
  randomly, and because of the nature of the training process we will describe later, we get two different models even if we perform two trainings in
  the very same way.

\medskip
2. 
{\it Validation set}: it is used to evaluate the performance of
  a model, that we have determined through a training.
  If the performance is not satisfactory, we change the form of the
  function that we have used for the training and we repeat the training thus
  determining a new model. The shape of the function depends on
  parameters, which are called \textit{hyperparameters} to distinguish
  them from the weights, which are determined during the training.
  The iterative process in which we determine the hyperparameters is called
  \textit{validation}. After every iteration we need to go back and
  obtain a new model via training.

\medskip
3. {\it Test sets}: it is used {\sl once} at the end of both
  training and validation to determine 
the accuracy of the algorithm, measured in percentage of accurate
predictions on total predictions regarding the labels of the test images.
It is extremely important to keep this dataset disjoint from training
and validation; even more important, to use it only once.

A typical partition of a labelled dataset is the following:
80\% training, 10\% validation, 10\% test. 

We now describe more in detail the key process in any supervised
classification task: the training.

\subsection{Training}\label{training-sec}

The expression training refers to the process of optimization
of the weights $w$ to produce 
a {\sl model} (see previous section).
In order to do this, we need three key ingredients: the {\sl score
  function}, the {\sl loss function} and the {\sl optimizer}.

\medskip
1. The \textit{score function} $s$ is assigning to each datum $x \in \R^d$, for example an image, and
to a given set of weights in $\R^p$ a \textit{score} for each class:
$$
\begin{array}{ccc}
s:\R^d \times \R^p &\lra & \R^C, \\ \\
(x,w) & \mapsto & s(x,w)=(s_0(x,w),\dots,s_{C-1}(x,w))
\end{array}
$$
where $C$ is the number of classes in our classification task
\footnote{Notice: we write $C-1$, since in python all arrays start from the
  zero index, so we adhere to the
  convention we see in the actual programming code.}
.
We can interpret the function $F: \R^d \lra \{0,\dots, C-1\}$, mentioned
earlier, associating to an image its class, in terms
of the score function. 
For a fixed model $w$,
$F$ is the function 
assigning, to an image $x \in \R^d$, the
  class corresponding to its highest score, i.e.
  $$
  F(x)= i, \qquad s_{i}(x,w) = \hbox{max} \{s_j(x,w), j=0, \dots, C-1\}
  $$
  The purpose of the training and validation is to obtain the best $F$
  to perform the classification task, which is, by no means, unique.

  \medskip

  To get an idea on the values of $p$, $d$, $C$
  we look at the example of the CIFAR10 database, where $C=10$.
  Then  if $x$ is an image of CIFAR10, $x \in \R^d$, for 
  $d=32\times 32 \times 3=3072$, while 
  $p$,  the dimension of the weight space, is typically in
the order of $10^6$ and depends on the choice of the score function.
We report in Table \ref{tab1} the dimension of the weight space $p$
for different architectures, i.e. choices of the score function,
on CIFAR10 (M$=10^6$).

\begin{table}[!h]
\centering
\caption{Values for $p$ for
various architectures on CIFAR10.}
\label{tab1}
\begin{tabular}{|l|l|}
\hline
Architecture &  $p=|$Weights$|$ \\
\hline
ResNet &  1.7M \\ 
Wide ResNet &  11M \\ 
DenseNet (k=12)  & 1M \\ 
DenseNet (k=24) & 27.2M  \\ 
\hline
\end{tabular}
\end{table}

We are going to see in our next section, focused on
Deep Learning, a typical form of
a score function in Deep Learning.

\medskip
2. The \textit{loss function} measures how accurate is the prediction
that we obtain via the score function. In fact, 
given a score function, 
we can immediately transform it into a probability distribution, by
assigning to an image $x$ the 
probability that $x$ belongs to one of the classes: 
\beq\label{prob-eq}
p_i(x,w):= \frac{e^{s_i(x,w)}}{\sum_j e^{s_j(x,w)}}
\eeq

If the probability 
in (\ref{prob-eq}) is a mass probability distribution concentrated
in the class corresponding to the correct label of $x$, we want the
loss function $L(x,w)$ for $x$ to be zero. We are going to see
in our next section a concrete example of such function.
Since the loss function is computed
via the score, it depends on both the weights $w$ and the data $x$. 
The total loss function (or loss function for short) is the sum of all the
loss functions for each datum, averaged on the number of data $N$:
$$
L: \R^p \lra \R, \qquad L(w)=\frac{1}{N} \sum_x L(x,w)
$$
For example in the MNIST dataset $N=60000$, hence the loss function
is the sum of $60000$ terms, one for each image in the training set.

\medskip
3. The \textit{optimizer} is the key to determine 
the set of weights $w$, which perform best on the training set.
As we shall presently see, at each step the weights will be
updated according to the optimizer prescription aimed to minimize (only
locally though) the loss function, that is the error committed in the prediction
of the labels on the samples in the training set.
In a problem of optimization, a standard technique to
minimize the loss function $L$ is the gradient descent (GD) 
with respect to the weights $w$.
In GD the update step in the space of parameters at time\footnote{The
  time is a discrete variable in machine learning, we increase it by
one at each iteration.} $t$
has the form:
\begin{equation}\label{gd}
    w_{t+1} = w_t - \eta\nabla L(w_t)\
\end{equation}
where $\eta$ is the \textit{learning rate}
and it is an hyperparameter  of the model, whose optimal value is chosen during
validation. The gradient is usually computed with an highly efficient algorithm called \emph{backpropagation} 
(see \cite{lecun2015}), where the chain rule, usually taught in calculus classes, plays a key role. We are unable
to describe it here, but for our synthetic exposition of key concepts, it is not essential to understand the
functioning of the algorithm.

The gradient descent minimization technique is guaranteed to reach a
suitable minimum of the loss function only if such function is convex.
Since the typical loss functions used in Deep Learning are non-convex,
it is more fruitful to use a variation of GD, namely
\textit{stochastic gradient descent} (SGD). This technique will add ``noise"
to the training dynamics, this fact is crucial for the correct performance,
and can effectively thermodynamically modeled
see \cite{cs}, 
for more details, we will not go into such description. 
SGD is one of the most
popular optimizers, together with variations on the same theme (like
Adam, mixing SGD with Nesterov momentum, see \cite{Adam}).
In SGD, we do not compute the gradient of the loss function
summing over all training data, but only on a randomly
chosen subset of training samples, called \textit{minibatch}, $\mathcal{B}$. 
The minibatch changes at each step, hence providing iteration after iteration
a statistically significant sample of all the training dataset. 
The update of the weights is obtained as in (\ref{gd}),
where we replace the gradient of the loss function  
$$
\nabla L(w) = 
{\frac{1}{N}} \sum_x 
\nabla L(x,w)  \quad \hbox{with} \quad 
\nabla_{\mathcal{B}} L(w) := \frac{1}{|\mathcal{B}|} \sum_{x \in \mathcal{B}}\nabla L(x,w)
$$ 
where $N$ is the size of the dataset (i.e. the number
of samples in the dataset), while $|\mathcal{B}|$ is the size
of the minibatch, another hyperparameter of the model. 
The SGD 
update step is then:
\begin{equation}\label{sgd}
      w_{t+1} = w_t - \eta\nabla_\mathcal{B} L({w_t})\
\end{equation}

In SGD the samples are extracted randomly and the same sample
can be extracted more than once in subsequent
steps for different minibatches.
This technique allows to explore the
{\sl landscape} of the loss function with a noisy movement around
the direction of steepest descent and also prevents the
training dynamics from being trapped in an inconvenient local
minimum of the loss function.

We define as \textit{epoch} the number of training steps
necessary to cover a number of samples equal to the size of
the entire training set:
$$
\mathrm{epoch}=\frac{N}{|\mathcal{B}|}
$$

For example in CIFAR10, where $N=60000$, if we choose a minibatch
size of 32, we have that an epoch is equal to $60000/32 =1875 $ iterations
i.e. updates of the weight as in (\ref{sgd}).
A typical training can last from 200-300 epochs
for a simple database like MNIST, to few thousands for a more
complicated one. Often the learning rate is modified during the
iterations, decreasing it as the training (also called learning
process) goes.
Usually, values of the loss and accuracies are printed out
every 10-100 epochs to check the loss values are actually decreasing
and the accuracy is improving. As usual, by accuracy we mean the number
of correct predictions with respect to the total number of data in training set.
Since in our examples, MNIST and CIFAR10
the labels are evenly distributed
among the classes, this notion of accuracy is meaningful.

\medskip
For a supervised classification task, let us summarize the training
process. We have the following steps:
\begin{itemize}
\item Step 0:
  Initialize randomly the weights.
\item Step 1:
  Assign a score to all data.

  \item Step 2:
  Compute the total loss based on score at Step 1.
\item Step 3: Update the weights according to the chosen optimizer
  (e.g. SGD) for a random minibatch.
\item Step 4: Repeat Steps 1, 2, 3 until the loss reaches
  a plateau and is not decreasing anymore (usually few hundreds epochs).
\end{itemize}

At the end of the training,
validation is necessary to understand whether a change
in our choices, like
the score function, the size of the minibatch and the learning rate,
can produce a better prediction, by computing the accuracy on
the validation set, consisting of images {\sl not} belonging to the
training set, that we have used to compute the optimal weight. In practice this means changing the hyperparameters of the model, repeating the training on the training set and then evaluating the trained model on the validation set to see if we get a better performance. This process is then repeated many times trying multiple hyperparameters combinations. There are techniques to conduct this \emph{hyperparameter optimization} in an efficient way, for example, the interested reader might read about one of these techniques here \cite{Optuna}.

\section{Deep Learning}\label{dl-sec}

In this section we shall focus on score and loss functions widely used when solving
supervised image classification tasks with Deep Learning algorithms. We will then focus on this specific task.
Again, this is not meant to be an exhaustive treatment, but just an overview
to help mathematicians and physicists to understand the mathematical and geometric
concepts underlying these algorithms.

\subsection{Score Function}\label{score-sec}

The purpose of the score function is to associate to every
image $x$ in our dataset, for a given set of weights $w \in \R^p$,
a score for each class:
$$
s:\R^d \times \R^p \lra \R^C
$$
where $d=b \times c$, if the image is black and white, $b$
and $c$ taking into account the dimensions of the image pixelwise.
If we have a color image, $d=3 \times b \times c$ because
of the three RBG channels (see also Sec. \ref{scd-sec}).
For clarity, we shall assume to have $d=b \times c$, the
general case being a small modification of this.

\medskip
The score function in supervised learning
characterizes the algorithm and for Deep Learning
it consists in the composition of several functions, each with
a geometrical meaning, in terms of images and perception.
We discuss some of the most common functions appearing in the
score function; this list is by no means exhaustive, but gives an
overview on how this key function is obtained. We also notice
that there is no predefined score function: for each
dataset some score functions
perform better than others. However very different score functions,
still within the scope of the Deep Learning,
may give comparable results in terms of accuracy.

\medskip
Common functions appearing in the expression of the score
function for Deep Learning algorithms are: linear and affine functions,
the RELU function, convolutional functions. We shall examine each
of them in some detail.

\medskip
\paragraph
{Affine functions.} 
Since
  any image is a matrix of integers, we can transform such matrix
  $x \in \R^{b} \times \R^{c}$ into a row vector in $\R^d$, $d=b\times c$
  by concatenating each row/column. Then, we define:
  $$
  \ell_{w,b}: \R^d \lra \R^n, \qquad \ell_{w,b}(x)=xw+b
  $$
  where $w=(w_{ij})$ is a $d \times n$ matrix of weights and
  $b \in \R^d$ a vector (also of weights) called \textit{bias}. 
  For simplicity, from now on we will work with linear functions only $\ell_w$, with
  associated matrix $w$, that is we take $b=0$, see
  App. \ref{mlp-app} for more details on the general case.

  \medskip
  If we define the score as:
  $$
  s:\R^d \times \R^{p} \lra \R^C, \qquad
  s(x,w)=\ell_w(x), \qquad p=d \times C
  $$
  then, we say we have a \textit{linear classifier} or
  equivalently we perform \textit{linear classification}. Linear classifiers do not belong
  to the family of Deep Learning algorithms, however they always appear in the set of functions
  whose composition gives the score in Deep Learning. 
  In this case, 
  the number
  of weights is $p=d \times C$.
  For the example of MNIST and linear classification, $W$
  is a matrix of size $p=784 \times 10$, since an image of MNIST
  is a black and white $28 \times 28$ pixel image and belongs to one
  of 10 classes each representing a digit.
  Linear classification on simple datasets as MNIST
  yields already significant results (i.e. accuracies above 85\%).
  It is important, however, to
  realize that linear classification is effective only if the
  dataset is \textit{linearly separable}, that is, we can partition
  the data space $\R^d$ using affine linear hyperplanes and all the
  samples of the same class lay into one component of the partition. 
\vskip 0.05in
\begin{figure}[ht]
\begin{center}
\centerline{\includegraphics[width=4.5in]{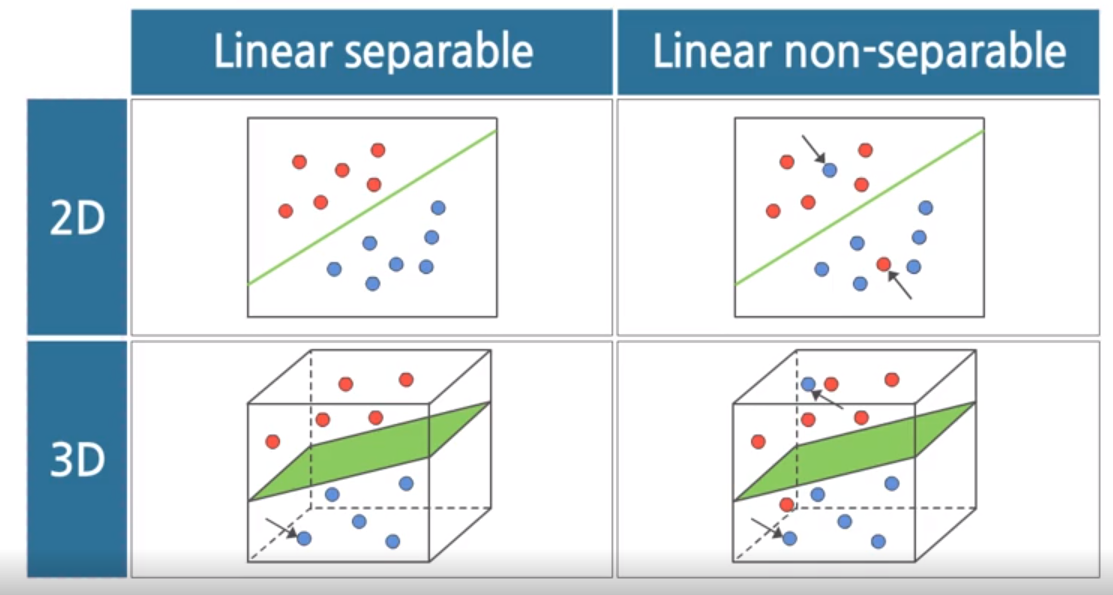}}
\vskip 0.05in
\caption{Linearly separable and linearly non-separable sets (github.org)}
\label{fig:lin}
\end{center}
\vskip -0.2in
\end{figure}
\paragraph{Rectified Linear Unit (RELU) function.} Let
the notation be as above. We define, in a generic euclidean
space $\R^m$, the RELU function as:
$$
\RELU: \R^m   \lra \R^m  , \qquad
\RELU(z)=(\max(0,z_1), \dots, \max(0,z_m)) 
$$
where $z=(z_1,\dots,z_m) \in \R^m$.

Notice that if the RELU appears in the expression of the score function,
it brings {\sl non linearity}.
If we realize the score function by composing 
linear and RELU functions, we can tackle more effectively the
classification task of a non linearly separable dataset (see Appendix \ref{Univ:approx}).
So we can define, for example, the non linear score function:
\beq\label{2layer}
s: \R^d \times \R^p \lra \R^C, \qquad s(x,w)=
(\ell_{w_2} \circ \RELU \circ \ell_{w_1})(x)
\eeq
This very simple score function already achieves a remarkable performance
of 98\% on the MNIST dataset, once the training and validation
are performed suitably (see Sec. \ref{training-sec}).
Notice that we have:
\beq\label{2layerbis}
\ell_{w_1}:\R^d \lra \R^h, \qquad \ell_{w_2}: \R^h \lra \R^C
\eeq
We know $d$ the data size (for example for MNIST $d=784$) and $C$
the number of classes (for example for MNIST $C=10)$),
but we do not know $h$, which is an hyperparameter
to be determined during validation (for example to achieve a 98\%
accuracy on MNIST, $h=500$).

\medskip
In the terminology of neural networks, that
we shall not fully discuss here, we say that $h$ is the
dimension of the \textit{hidden layer}.
So, we have defined in (\ref{2layer})
a neural
network consisting of the following layers:
\begin{itemize}
\item the \textit{input layer},
defined by the linear function $\ell_{w_1}$ and
taking as input the image $x$;
\item the
\textit{hidden layer} defined by
$\ell_{w_2}$ taking as input the vector RELU$(\ell_{w_1}(x))$; 
\item the
  \textit{output layer} providing a score for each class.
  \end{itemize}
We say that the score in (\ref{2layer})
defines a \textit{two layer network}, since we do not count (by
convention) the input layer\footnote{Some authors have RELU
  count as a layer, we do not adhere here to such convention.}.
In Fig. \ref{fig:2layer} we give schematically the layers
by nodes, each node corresponding to the coordinate of the vector
the layer takes as input; in the picture below, we have $d=3$, $h=4$, $C=2$.

\begin{figure}
    \centering
    \includegraphics[width=0.45\textwidth]{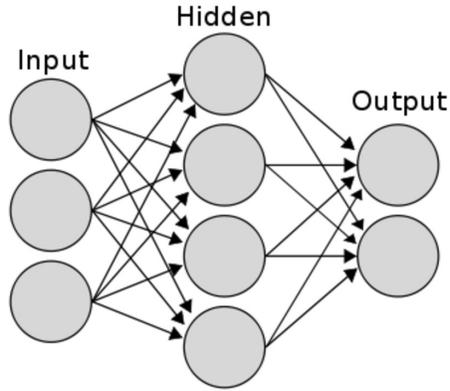}
    \caption{Representation of a two layer network}
    \label{fig:2layer}
\end{figure}

Notice that such score function is evidently not differentiable,
while we would need it to be, in order to apply the stochastic gradient descent
algorithm. Such problem is solved with some numerical tricks, we do not detail here.

\medskip
The supervised classification Deep Learning algorithm,
where the score function is expressed as the composition
of several affine layers with RELU functions between them as in (\ref{2layer}),
takes the name of \textit{multilayer perceptron} (see also Appendix \ref{mlp-app}
for a variation with convolution functions).
In Appendix \ref{mlp-app} we discuss the mathematical significance
of such score function in classification and regression tasks.

\paragraph{Convolution functions.} 
These functions characterize a particular type of Deep Learning networks, called \emph{Convolutional Neural Networks} (CNNs),
though some authors identify Deep Learning with CNNs.
We shall describe briefly one
dimensional convolutions, though image classification requires two dimensional
ones, see the Appendix \ref{mlp-app} for more details.

One dimensional convolutions arise as discrete
analogues of convolutions in mathematical analysis, more specifically with an integral kernel (integral transform).
Let us consider a vector $x \in \R^d$. We can define a simple one dimensional convolution as follows:
  $$
  \mathrm{conv1d}:\R^d \lra \R^{d-r},\quad
  (\mathrm{conv1d}(x))_i=\sum_{j=1}^{r} K_{j}x_{i+j}
  $$
  where $K$ is called a \textit{filter} and is a vector of weights
  in $\R^r$, $r<d$. 
We shall describe in detail and in a more general way two (and higher) dimensional
convolutions in the Appendix \ref{mlp-app}. 
CNNs networks  typically consist of several layers of convolutions, followed by few linear layers, 
separated by RELU functions.
  
\subsection{Loss Function}\label{loss-subsec}

The loss function for a given image $x$  measures how accurately the algorithm
is performing its prediction, once a score $s(x,w)$ is
assigned to $x$, for given set of weights $w$. 
A popular choice for the loss
function is the \textit{cross entropy loss}, defined as follows:
$$
L(x,w)= -  \log{\frac{e^{s_{y_x}(x,w)}}{ \sum_{j=1}^C e^{s_j(x,w)}}}
$$
where $C$ is the number of classes in our classification task and $y_x$ is the
label assigned to the datum $x$, that is a number between $0$ and $C-1$, corresponding to the class of $x$ 
(in our main examples, $x$ is an image). 

\medskip
In mathematics the
\textit{cross-entropy} of the probability distribution $p$ 
relative to the probability distribution $q$ 
is defined as:
$$
H ( q , p ) = - \bE_q[\log p] =-\sum_{i=1}^C q_i
\log p_i
$$
where $\bE_q[\log(p)]$ denotes the expected value of the logarithm of 
$p$ according
to the distribution $q$ and the last equality is the very definition of it
in the case of discrete and finite probability distributions.

In our setting, $q$ represents the ground truth distribution and
it is a probability mass distribution: 
\beq\label{qmass}
q_i(x)=\left\{
\begin{array}{cc} 1  & \hbox{if} \, x \, \hbox{has label} \, i \\
0  & \hbox{otherwise}
\end{array} \right.
\eeq
and depends on $x$ only.
The distribution $p$ is obtained from the score function $s$ via
the \textit{softmax} function $\mathcal{S}$, that is:
$$
p_i(x,w)= -  \log{\frac{e^{s_{y_x}(x,w)}}{ \sum_{j=1}^C e^{s_j(x,w)}}}
$$
where
$$
\mathcal{S}:\R^C \lra \R^C, \qquad
(\mathcal{S}(z_1\dots z_C))_i=\frac{e^{z_i}}{ \sum_{j=1}^C e^{z_j}}
$$
Notice that $p$ depends on both the image $x$ and the weights $w$.
With such $q$ and $p$ we immediately see that:
$$
\begin{array}{lr}
H ( q(x) , p(x,w) ) &= - \bE_q[ \log p ] = -\sum_{i=1}^C q_i(x)
\log p_i(x,w)= \\ \\
&=-\sum_{i=1}^C q_i(x) \log{\frac{e^{s_{i}(x,w)}}{ \sum_{j=1}^C e^{s_j(x,w)}}}=L(x,w)
\end{array}
$$
since $q_i(x)=1$ for $i=y_x$ and $q_i(x)=0$ otherwise.

In our setting the cross entropy of $q$ and $p$ coincides with
the Kullback-Leibler divergence. In fact, in general if $q$ and $p$
are two (discrete) probability distributions, we define their
Kullback-Leibler divergence as:
$$
\KL(q||p)=\sum_i q_i \log \frac{q_i}{p_i}
$$
Notice that $\KL(q||p)$ gives a measure on how much the probability
distributions $q$ and $p$ differ from each other.
We can write:
$$
\KL(q||p)=\sum_i q_i \log \frac{q_i}{p_i}=
\sum_i q_i\log(q_i) - \sum_i q_i\log(p_i)=
H(q)+H(q,p)
$$
where $H(q)=\sum_i q_i\log(q_i)$ is the \textit{Shannon entropy}
of the distribution $q$. $H(q)$ measures how much the distribution $q$ is 
{\sl spread}; it is zero for a mass probability distribution.
Hence in our setting, where $q(x)$ is as in (\ref{qmass}) we have
$$
\KL(q(x)||p(x,w))=H(q(x),p(x,w))=L(x,w)
$$

For a link to the important field of information geometry see
our Appendix \ref{infogeo}.

\section{Geometric Deep Learning: Graph Neural Networks} \label{gdl-sec}

Geometric Deep Learning (GDL), in its broader sense, is a novel and emerging field in machine
learning (\cite{bronstein2016, bronstein2021} and refs therein) 
employing deep learning techniques on datasets that are organized according to some underlying geometrical structure, for example the one of a graph. Convolutional Neural Networks are already part of GDL, as they are applied to some data that is organized in grids, as it happens for images, represented by matrices, where every entry is a pixel,
hence in a grid-like form.
When the data is organized on a graph, we enter the field of the so called \emph{Graph Neural Networks} (GNNs), deep learning algorithms that are built to leverage the geometric information underlying this type of data. In this paper we shall focus on GNNs for supervised node classification only, as we shall make precise later.
Data organized on graphs 
are commonly referred to as {\sl non-euclidean data}\footnote{Notice
  that since all graphs are embeddable into
  euclidean space (see \cite{godsil}), this terminology is in slight conflict with the
  mathematical one.}.
Data coming with a graph structure are ubiquitous,
from graph meshes in 3D imaging, to biological and chemical datasets: these data have usually the form of one or more graphs having a vector of features for each vertex, or node.
Because of the abundance of datasets of this form, as the graph structure usually carries important information usually ignored by standard Deep Learning algorithms (i.e. algorithms that would consider the node features only to solve the tasks at hand), Graph Neural Networks carry a great potential
for more applications, than just Deep Learning.

We start with a very brief introduction on graph theory, that may be
very well skipped by mathematicians familiar with it, then we will
concentrate on some score functions for Graph Neural Networks, discussing concrete examples.

\subsection{Graphs}\label{graph-sec}

We give here a very brief introduction to graphs and their
main properties. For more details see \cite{godsil}.

\begin{definition}
A \textit{directed graph} is a pair $G=(V, E)$, where $V$ is a finite set 
$V = {v_1, \dots, v_n}$ and
$E \subset V \times V$\footnote{We shall content ourselves to consider graphs having at most one vertex in each direction between two nodes.}. The elements of $V$ are called {\it nodes} or 
{\it vertices}, while the elements of $E$ are called {\it edges}.
If $(v_i,v_j)$ is an edge, we call $v_i$ the \textit{tail} and $v_j$
the {\it head} of the edge. For each edge $(v_i,v_j)$ we assume $i \neq j$, that is there are no
\textit{self loops}.

\medskip
We represent a graph in the plane
by associating points to vertices and arrows to
edges, the head and tail of an edge corresponding to the head and tail
of the arrow and
drawn accordingly as
in Fig \ref{fig:graph-ex}. We also denote the edge $(v_i,v_j)$ with $e_{ij}$.
A graph is {\it undirected}
if, whenever $e_{ij}
\in E$, we have that also $e_{ji}
\in E$.  
\end{definition}


\medskip
When the graph is undirected, we do not draw arrows, just links
corresponding to vertices
as in Fig.
\ref{fig:graph-ex}.

\vskip 0.05in
\begin{figure}[ht]
\begin{center}
\centerline{\includegraphics[width=5in]{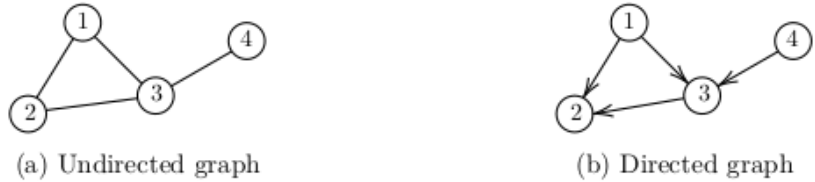}}
\vskip 0.05in
\caption{Directed and unidirected graph examples}
\label{fig:graph-ex}
\end{center}
\end{figure}
\vskip 0.05in

\begin{definition}
  The \emph{adjacency matrix $A$} of a
  graph $G$ is a $|V| \times |V|$
  matrix defined as
\begin{equation*}
    A_{ij} = 
    \begin{cases}
    1 & e_{ij} \in E \\
    0 & \text{otherwise}
    \end{cases}
\end{equation*}
\end{definition}
Notice that in an undirected graph $a_{ij}=1$ if and only if $a_{ji}=1$,
hence $A$ is symmetric.
Let us see the adjacency matrices for the graphs
shown in Fig. \ref{fig:graph-ex}.

\begin{figure}[h]
\begin{minipage}[b]{0.45\linewidth}
\centering
    \begin{equation*}
    A = 
    \begin{pmatrix}
    0 & 1 & 1 & 0 \\
    1 & 0 & 1 & 0 \\
    1 & 1 & 0 & 1 \\
    0 & 0 & 1 & 0 
    \end{pmatrix}
    \end{equation*}
    \caption{Adjacency matrix for the undirected graph in Fig.
      \ref{fig:graph-ex}.}
\label{fig:A_undirected}
\end{minipage}
\hspace{0.5cm}
\begin{minipage}[b]{0.45\linewidth}
\centering
    \begin{equation*}
    A = 
    \begin{pmatrix}
    0 & 1 & 1 & 0 \\
    0 & 0 & 0 & 0 \\
    0 & 1 & 0 & 0 \\
    0 & 0 & 1 & 0 
    \end{pmatrix}
    \end{equation*}
\caption{Adjacency matrix for the directed graph in Fig. \ref{fig:graph-ex}.}
\label{fig:A_directed}
\end{minipage}
\end{figure}

In machine learning is also important to consider the case of a
\textit{weighted adjacency matrix} $W=(w_{ij})$ associated to a graph:
it is a $|{V}| \times |{V}|$ matrix with real entries,
with $w_{ij} \neq 0$ if and only if $e_{ij}$ is an edge.
So, effectively, it is a way to associate to an edge a weight.
Similarly, we can also define the \textit{node weight matrix}, it is a
$|{V}| \times |{V}|$ diagonal matrix and allows us to associate
a weight to every node of a graph.

\begin{definition}
  Given an undirected graph $G=(V,E)$, we define the \emph{node degree}
  of the node $v$ as the number of edges connected with $v$. 
  The \emph{degree matrix $D$} 
  is a
  $|V| \times |V|$ diagonal matrix with
  the degree of each node on its diagonal, namely $D_{ii} = \deg(v_i):=\sum_j A_{ij}.$
We refer to the set of vertices connected with a node $v_i$ as the
\emph{neighbourhood $\mathcal{N}\left( v_i \right)$} of the vertex $v_i$.
\end{definition}

The concept of node degree
and node neighbourhood
will turn out to be very important in Graph Neural Networks.

\medskip
We now introduce the {\sl incidence matrix}.

\medskip
Let $G=(V,E)$ be a graph.
Let $C(V)$ denote the vector space of real valued functions on
the set of vertices $V$ and $C(E)$ the vector space of real valued
functions on edges:
$$
C(V)=\{f:V \lra \R\}, \qquad C(E)=\{g:E \lra \R\}
  $$
  We immediately have the following identifications:
  \beq\label{ident}
  C(V) = \R^{|V|}, \quad C(E) = \R^{|E|}
  \eeq
  obtained, for $C(V)$,
  by associating to a function a vector whose $i^\mathrm{th}$ coordinate
  is the value of $f$ on the $i^\mathrm{th}$ vertex; 
  the case of $C(E)$ being the same.

\begin{definition}
  Let $G=(V,E)$ be a directed graph and $f \in C(V)$. We define the
  \emph{differential} $df$ of $f$ as the function:
\begin{align*}
    df \colon E & \longrightarrow \mathbb{R} \\
    e_{ij} & \longmapsto f(v_j) - f(v_i) .
\end{align*}
We define the \textit{differential} as
$$
d: C(V) \lra C(E), \qquad f \mapsto df
$$
\end{definition}
We can think of $d$ as a boundary operator in cohomology, viewing the
graph as a simplicial complex. We will not need this interpretation
in the sequel.

\medskip
It is not difficult to verify that in the natural identification (\ref{ident})
the differential is represented by the matrix:
\begin{equation*}
    X_{ij} = 
    \begin{cases}
    -1 & \text{if the edge $v_j$ is the tail of the edge $e_i$} \\
    1 & \text{if the edge $v_j$ is the head of the edge $e_i$} \\ 
    0 & \text{otherwise}
    \end{cases}.
\end{equation*}
Notice that $X$ row indeces correspond to edges, while the column ones to vertices
(we number edges here regardless of vertices). 

In other words if $f=(f_i)$ is a function on $V$, identified with
a vector in $\R^{|V|}$, we have:
$$
(df)_k=\sum_{i=1}^{|V|} X_{ki}f_i, \qquad df=((df)_k) \in \R^{|E|}
$$

The matrix $X$ is called the \textit{incidence matrix} of the directed graph
$G=(V,E)$. Let us see an example to clarify the above statement.

\begin{example}
  For the graph in Fig. \ref{fig:A_undirected},  we have 
\begin{equation*}
 df=   \begin{pmatrix}
    -1 & 1 & 0 & 0 \\
    -1 & 0 & 1 & 0 \\
    0 & 1 & -1 & 0 \\
    0 & 0 & 1 & -1 
    \end{pmatrix}
    \begin{pmatrix}
      f(v_1) \\
          f(v_2) \\
    f(v_3) \\
    f(v_4) 
    \end{pmatrix} = 
    \begin{pmatrix}
    - f(v_1) + f(v_2) \\
    - f(v_1) + f(v_3) \\
    f(v_2) - f(v_3) \\
    f(v_3) - f(v_4)
    \end{pmatrix} = 
    \begin{pmatrix}
    df(e_{1}) \\
    df(e_{2}) \\
    df(e_{3}) \\
    df(e_{4}) 
    \end{pmatrix}
\end{equation*}
\end{example}

Notice that, suggestively, many authors, including \cite{bronstein2016} write the symbol
$\nabla$ for $X$, because this is effectively the discrete analogue of the
gradient in differential geometry. 

\subsection{Laplacian on graphs}

In this section we define the {\sl laplacian} on undirected graphs
according to the reference \cite{bronstein2016}. Notice that there are other
definitions, see \cite{bauer, godsil}, but we adhere to \cite{bronstein2016}, because
of its Geometric Deep Learning significance.

We start with the definition of laplacian matrix, 
where, for the moment, we limit ourselves to the case in
which both edges and nodes have weight 1 (see Sec. \ref{graph-sec}).
We shall then relate the laplacian matrix with the incidence matrix
and arrive to a more conceptual definition in terms of the differential.

\begin{definition}
\label{laplacian}
The \emph{laplacian matrix $L$} of an undirected graph is a
$|{V}| \times |{V}|$ matrix, defined as
\begin{equation*}
    L \coloneqq D - A,
\end{equation*}
where $A$ and $D$ are the adjacency matrix and degree matrix, respectively.
\end{definition}

We now introduce the concept of \textit{orientation}.

\begin{definition}
  An \textit{orientation} of an undirected graph is an
  assignment of a direction to each edge, turning the
  initial graph into a directed graph.
\end{definition}

\begin{figure}[h!]
    \centering
    \includegraphics[width=0.55\textwidth]{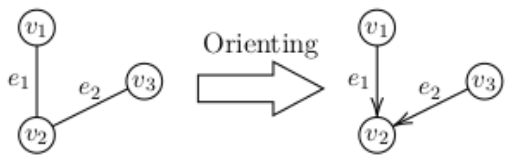}
    \caption{An orientation on a graph}
    \label{fig:orient}
\end{figure}

The following result links the laplacian on graph with
the incidence matrix.

\begin{proposition}\label{lapl-prop}
  Let $G=(V,E)$ be an undirected graph and let us fix an orientation.
  Then: 
  $$
  L=X^tX
  $$
  where $X$ is the incidence matrix associated with the given orientation.
\end{proposition}

Notice that this result is independent from the chosen orientation.
The proof is not difficult and we leave it to reader; the next example
will clearly show the argument.

  \begin{example}
    Let us consider the graph of Fig. \ref{fig:orient}.
We can write the Laplacian according to the definition as
\begin{equation*}
    L = D - A = 
    \begin{pmatrix}
    1 & -1 & 0 \\
    -1 & 2 & -1 \\
    0 & -1 & 1
    \end{pmatrix}.
\end{equation*}
This is the same as writing $L$ as 
\begin{equation*}
    L = X^t X = 
    \begin{pmatrix}
    -1 & 0 \\
    1 & 1 \\
    0 & -1 
    \end{pmatrix}
    \begin{pmatrix}
    -1 & 1 & 0 \\
    0 & 1 & -1 
    \end{pmatrix} = 
    \begin{pmatrix}
    1 & -1 & 0 \\
    -1 & 2 & -1 \\
    0 & -1 & 1
    \end{pmatrix}.
\end{equation*}
\end{example}

  \begin{remark}
    We warn the reader about possible confusion in the notation present
    in the literature.
    \begin{enumerate}
    \item The incidence matrix of an oriented undirected graph is not the
      same as the incidence matrix of an undirected graph
      (that we have not defined in here) as one can find it for
      example in \cite{godsil}.
    \item The laplacian of a directed graph, as defined for example
      in \cite{bauer}, is not the laplacian of the undirected graph with an
      orientation and it is not expressible, at least directly, in terms of
      its incidence matrix.
    \end{enumerate}

    Originally in the pioneering works \cite{bronstein2016} and \cite{kw} the
    authors discuss undirected graphs only and this is the reason why we
    have limited our discussion to those, though in practical applications
    directed graphs may turn to be quite useful.
    \end{remark}

  We end this section with a comment regarding a more intrinsec definition
  of laplacian.

  \begin{definition}
    Let $G=(V,E)$ be an undirected graph with an orientation
    and let $C(V)$, $C(E)$ the
    functions on $V$ and $E$ as above. We define the \textit{laplacian operator} as:
    $$
    \mathcal{L}: C(V) \lra C(V), \qquad \mathcal{L}=d^* d
$$
    where $d^*:C(E)^* \lra C(V)^*$ is the dual morphism and we identify
    $C(E)$ and $C(V)$ with their duals with respect to their canonical bases.
  \end{definition}

  We leave to the reader the easy check that, once we use the matrix
  expression for $d$ and $d^*$, that is $X$ and $X^t$,
  we obtain the same expression for
  the laplacian as in
  Prop. \ref{lapl-prop}.

  \begin{remark} Let $G=(V,E)$ be an undirected graph with an orientation.
    In the suggestive interpretation of the incidence matrix as the
    discrete version of the
    gradient of a function $f$ on vertices, we have the correspondence:
    $$
    \begin{array}{c}
    \mathcal{F}:\R^n \lra \R, \\ \\ \nabla(\mathcal{F})= \begin{pmatrix} \partial_{x_1}\mathcal{F} \\ \vdots
      \\  \partial_{x_n}\mathcal{F}  \end{pmatrix} 
  \end{array} \iff \qquad \begin{array}{c} f:V \lra \R \\ \\
      X({f})=\begin{pmatrix}
      {f}(v_{j_1})-{f}(v_{i_1}) \\ \vdots \\
      f(v_{j_n})-f(v_{i_n}) \end{pmatrix}\end{array}
      $$
      for edges $(i_1,j_1)$, $\dots$, $(i_n,j_n)$ and where $\mathcal{F}$ denotes a smooth
      function on $\R^n$.
      If we furtherly
      extend this correspondence 
      to the laplacian operator:
    $$
    \Delta(\mathcal{F})=(\nabla^t \nabla)(\mathcal{F})=\partial^2_{x_1}\mathcal{F} + \dots +
      \partial_{x_n}^2\mathcal{F}  \quad \iff \quad L(f)=(X^tX)(f)
      $$
      thus justifying the terminology ``laplacian'' for the operator
      $L$ we have defined on $C(V)$.
  \end{remark}

  In \cite{bronstein} this correspondence is furtherly exploited to
  perform spectral theory on graphs, introducing the discrete analogues
  of divergence, Dirichlet energy and Fourier transform. We do not
  pursue further this topic here, sending the interested reader to
  \cite{bronstein2016} for further details.

  \medskip
  We also mention that a more general expression of $L$ is obtained
  by writing:
  $$
  L_g=W_v(D-W_a)
  $$
  where $W_v$ is a diagonal matrix attributing a weight to each
  vertex and $W_a$ is the weighted adjacency matrix. So $L_g=L$
  if we take $W_v=I$ and $W_a=A$.

  A remarkable case occurs for $W_v=D^{-1}$ and $W_a=A$:
  $$
  L_d:=D^{-1}(D-A)=I-D^{-1}A
  $$
  in this case we speak of \textit{diffusive (or normalized) laplacian}, since multiplying
  $L$ by the inverse of the node degree matrix amounts
  in some sense to take the average of each coordinate, associated to a node,
  by the number of links of that node. We will come back to $L_d$ in our
  next section.
  
\subsection{Heat equation}\label{heat-sec}

Deep Learning success in supervised classification tasks
is due, among other factors, to the convolutional layers, whose filters (i.e. kernels)
enable an effective pattern recognition, for example in image classification
tasks.

In translating and adapting
the algorithms of Deep Learning to graph datasets, it is
clear that convolutions and filters cannot be defined in the same way,
due to the topological difference of node neighbours: different nodes
may have very different neighbourhoods, thus making the definition
of convolution difficult.

In Graph Neural Networks, the convolutions are replaced
by the {\sl message passing mechanism} in the {\sl encoder}
that we shall discuss in the sequel.
Through this mechanism
the information from one node is {\sl diffused} in a way that
resembles heat diffusion in the Fourier heat
equation, as noticed in \cite{bronstein2016}.

We recall that the classical heat equation, with
no initial condition, states the following:
\begin{equation}
\label{eq:heat_eq_classic}
    \partial_t h(x,t) = c \Delta h(x,t) \\
\end{equation}
where $h(x,t)$ is the temperature at point $x$ at time
$t$ and $\Delta$ is the laplacian operator, while
$c$ is the \emph{thermal diffusivity constant}, that we set equal to -1.

\medskip
We now write the analogue of such equation in the graph
context, following \cite{bronstein2016}, where now $x$ is replaced by
a node in an undirected, but oriented, graph and $h$ is a function on the vertices
of the graph, depending also on the time $t$, which
is discrete. We then write:

\begin{equation*}
  h_{t+1}(v) - h_t(v) =
-  L_d(h_t(v))
\end{equation*}
where $L_d$ is the \textit{diffusive laplacian}.
We now substitute the expression for the laplacian $L_d$ obtaining:
\beq\label{heat-eq}
h_{t+1}(v) = h_t(v) - [(I-D^{-1}A)h_t](v)= (D^{-1}A h_t)(v)=
\sum_{u \in \cN(v)} \frac{h_t(u)}{\deg(v)}
\eeq
where the last equality is a simple exercise.
Hence:
\beq
h_{t+1}(v) = \sum_{u \in \mathcal{N}(v)} \frac{h_t(u)}{\deg(v)}
\eeq
where $\mathcal{N}(v)$ denotes the neighbourhood of the vertex $v$.

We are going to exploit fully this analogy in the description
of graph convolutional networks in the next section.

\subsection{Supervised Classification on Graphs} 

In many applications, data is
coming naturally equipped with a graph structure.
For example, we can view 
a social network as a graph, where each node is a user and 
edges are representing, for example, the friendship between two users.
Also, in chemistry applications, we may have a graph associated to
a chemical compound and a label provided for each graph
classifying the compound. For example, we can
have a database consisting of protein molecules,
given the label zero or one depending on
whether or not the protein is an enzyme. 

\medskip
Supervised classification tasks on graphs can be roughly 
divided into the following categories: 
\begin{itemize}
\item {\it node classification}: given a graph with a set of features for each node and a labelling on
the nodes, find a function that given the graph structure and the node features predicts the labels on the nodes;
\item {\it edge classification}: given a graph with a set of features for each node and a labelling on
the edges (like the edge exists or not), find a function that given the graph structure and the node features predicts the labels on the edges\footnote{This problem requires care, as for the training we may need to modify the graph structure.};
\item {\it graph classification}: given a number of graphs with features on the nodes, possibly external global features and a label for
each one, find a function to predict the label of these graphs (and similar ones).
\end{itemize}
For the node classification tasks, the split train/valid/test is performed on the nodes, i.e. only a small subset of the nodes is used to train the score function and only their labels intervene to compute the loss, as we will see. The entire graph structure and data is however available during training as GNNs typically leverage the network topological structure to extract information. 
Therefore not only the information of the nodes available for training is used, in the sense that also the features, \emph{but not the labels}, of the nodes not used in the training are a priori available during the training process. For this reason, in this context many authors speak of \emph{semi-supervised} learning. There are also other classification tasks, 
that we do not mention here, see \cite{bronstein2016} for more details. %
For clarity's sake we focus on the node classification task and the Geometric Deep Learning algorithms we will describe will fall in the category of \emph{Graph Neural Networks} (GNNs).
We shall also examine a key example,
the Zachary Karate Club (\cite{kw}) dataset, to elucidate the theory.

\medskip
The ingredients for node classification with Graph Neural Networks are the same as for
image classification in Deep Learning (see Sec. \ref{training-sec}), namely:
\begin{itemize}
\item 
the {\sl score function} assigning to each node with given features and
to a given set of weights in $\R^p$ a score for each class;
\item the {\sl loss function} measuring how accurate is the prediction 
obtained via the score function;
\item the {\sl optimizer} necessary to determine the set of weights, which
perform best on the training set.
  \end{itemize}

For Graph Neural Networks loss function and optimizer are
chosen in a very similar way as in Deep Learning and the optimization
procedure is similar. 
The most important difference occurs
in the score function, hence we focus on its description. In 
Deep Learning for a set of weights $w$, the score function $s(x,w)$ 
assigns to the datum 
$x\in\R^d$ 
(e.g. image), 
the vector of scores. 
On the other hand, in GNNs, the topological information regarding the graph plays a key role. 
The vector of features $x\in\R^d$ of a node is 
\emph{connected with the others according to a graph structure} and therefore the information of the graph $G$ we are considering, represented by its adjacency matrix $A_G$, needs to be given as input as well. This dependence is often made explicit by denoting the resulting score functions as $s(x,w,A_G)$. 
Let us see a concrete example in more detail in the next section.

\subsection{The Score function in Graph Neural Networks}

In the supervised node classification problem, we assume to have
an undirected graph $G=(V,E)$ and for each node of the graph
we have the following two key information:
the \textit{features} associated to the node and
the \textit{label} of the node, which is, as in ordinary Deep Learning,
a natural number between $0$ and $C-1$, $C$ being the
number of classes.

We can view the features of the nodes as a vector valued function
$h$ on the nodes, hence $h:V\lra \R^n$ if we have $n$ features. 
Hence $h \in C(V)^n=(\R^{|V|})^n\cong\R^{|V|\times n}$
and we denote $h(v)$ or $h_v\in\R^n$, the value of the feature $h$ at the
node $v$. When $n=1$ this is 
equivalent to the $v^{\mathrm{th}}$ 
coordinate of $\R^{|V|}$.
The $n$ dimensions should be thought as $n$ \emph{feature channels} according to the philosophy introduced in \cite{bodnar2022} or \cite{SheafNN}.

The score function in the GNNs we shall consider consists
of two distinct parts:

\begin{enumerate}
\item the {\it encoder}, which produces an embedding of the graph features
  in a latent space:
  $$
  E: \R^{|V|\times n} \lra \R^{|V|\times c},
  $$
\item the {\it decoder}, which predicts the score of each
  node based on the embedding obtained via the encoder.
  This is usually a multilayer perceptron
  (see Sec. \ref{score-sec}) or even just a linear classifier.
  \end{enumerate}
We therefore focus on the description of the encoder only.

\medskip
The encoder implements what is called the \textit{message passing
  mechanism}. There are several ways to define it, we give
one simple implementation consisting in a sequence of graph convolutions described by equation (\ref{eq:convolution_aggregator}) below, the others being a variation on this theme.
We start with the initial feature $h_v^0$ at each node $v$ and we define
the next feature recursively as follows:
\begin{equation}
\label{eq:convolution_aggregator}
h_v^{k} = \sigma
\left(W_k \sum _{u\in \mathcal{N}(v)} \frac{{h}_u^{k-1}}{\deg (v)} +B_k\cdot{h}_v^{k-1}\right) \qquad k=1,\dots, \ell
\end{equation}
where $\ell$ is the number of layers in the encoder, $\sigma$ is a generic
non linearity, for example the RELU and $W_k$, $B_k$ are matrices of weights of the appropriate size,
which are determined during the training. We say that equation (\ref{eq:convolution_aggregator}) defines a message passing mechanism. Training typically takes
place with a stochastic gradient descent, looking for a local minimum
of the loss function.
Notice that the equation above reminds us of the heat equation (\ref{heat-eq}) we introduced in the context of graphs having scalars as node features.
Indeed, in the case all $B_k=0$ for all $k$, if we denote as $H_{k}\in\R^{|V|\times p}$ the matrix having as rows all node features after having applied the first $k$ layers, the 
$k$-step of the encoder in (\ref{eq:convolution_aggregator}) above, can be written concisely, for all the node features together, as $$ 
H^k=\sigma(D^{-1}A_GH^{k-1}W_k^t)$$ where $D$ is the diagonal matrix of degrees and $A_G$ is the adjacency matrix of $G$. As a consequence, the encoder can be thought, modulo non linearities, as a "1-step Euler discretization" of a PDE of the form $\dot{H}=-\Delta H(t)$, $ H(0)=H_0$ i.e. as following the discrete equation $$H(t+1)-H(t)=-L_dH(t)$$

where we have denoted as $L_d$ the diffusive laplacian of Section \ref{heat-sec} and $H_0$ is the starting datum of node features. By modifying equation (\ref{eq:convolution_aggregator}) we obtain some notable graph convolutions:
\begin{itemize}
    \item If $B_k=W_k$ we obtain the influential Kipf and Welling graph convolution \cite{kw} that can be written more concisely as 
    \begin{equation}\label{kwconv}
        f(H_{k},A_G)=\sigma(D^{-1}\widehat{A_G}H_{k-1}W_{k}^t)
    \end{equation}
    where $\widehat{A_G}=A_G+Id$. This convolution can be seen to arise as a 1-step Euler discretization of an heat equation on a graph where one self loop is added to each vertex.
    \item If we do not normalize by the degress, equation \ref{eq:convolution_aggregator} gives us an instance of the influential GraphSAGE operator from \cite{GRLHamilton}.
\end{itemize}

In the next section we give a concrete example of the message
passing mechanism and how the encoder and decoder work.

\subsection{The Zachary Karate Club}\label{zach-sec}

The Zachary Karate club is a social network deriving from a 
real life karate club studied by W. Zachary
in the article \cite{kw} and can be currently regarded
as a toy benchmark dataset for Geometric Deep Learning.
The dataset consists in 34 members of the 
karate club, represented by the nodes of a graph, with 
154 links between nodes, corresponding 
to pairs of members, who interacted also outside the club. 
Between 1970 and 1972, the club became divided into
four smaller groups of people, over 
an internal issue regarding the price of
karate lessons. Every node is thus labeled by an integer
between 0 and 3,
corresponding to the opinion of each member of the club, with respect
to the karate lesson cost issue (see Fig. \ref{fig:karateclub}).
Since the dataset comes with no features (i.e. is \emph{feature-less}) associated to nodes, but
only labels, we assign (arbitrarily) 
the initial matrix of features to be the identity i.e. $H_0=\mathrm{Id}_{34}$.

\begin{figure}[h!]
    \centering
    \includegraphics[width=0.75\textwidth]{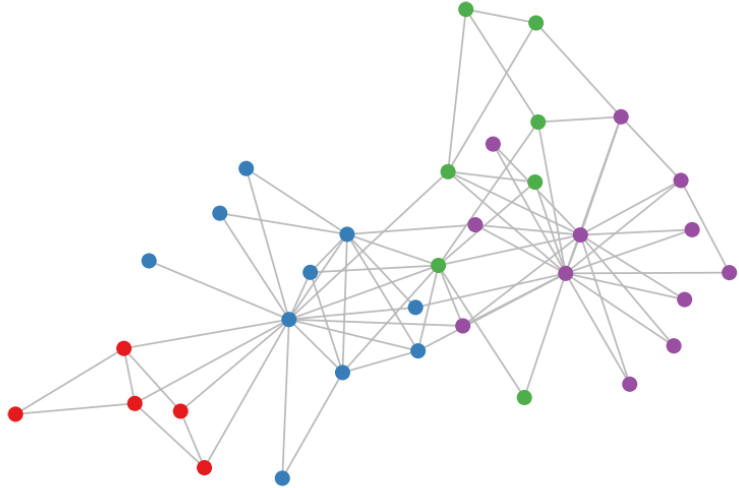}
    \caption{Representation of the Zachary Karate Club
      dataset (see \cite{kw}). 
      Different colors correspond
      to different labels.}
    \label{fig:karateclub}
\end{figure}

\begin{figure}[h!]
    \centering
    \includegraphics[width=0.75\textwidth]{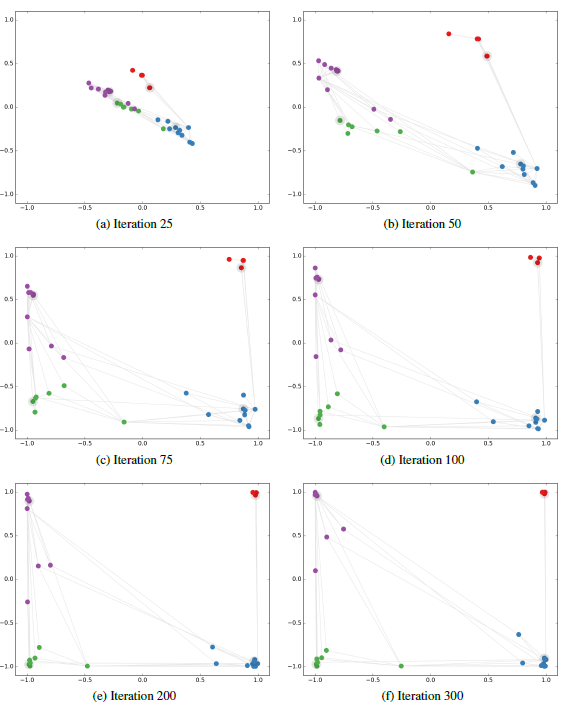}
      \caption{Representation of node embeddings at different epochs (\cite{kw}).
 Different colors correspond
      to different labels.   }
    \label{fig:4_embedding}
\end{figure}

We represent schematically below a 3 layers GNN, used in \cite{kw} on this dataset the encoder-decoder framework,
reproducing at each step the equation
(\ref{kwconv})\footnote{A particular case of (\ref{eq:convolution_aggregator}) with $l=3$}. In other
words we have functions (conv1), (conv2), (conv3), given by
the formula (\ref{kwconv}) for weight matrices $W_k$
of the appropriate size, with $\sigma=tanh$

\medskip
  
$\left.
\begin{tabular}{l}
\qquad  (conv1): $\R^{34} \lra \R^4$,  $h_v^{1} \mapsto h_v^2$ \\ 
\qquad  (conv2): $\R^{4} \lra \R^4$, $h_v^{2} \mapsto h_v^{3}$\\  
\qquad  (conv3): $\R^{4} \lra \R^2$, $h_v^{3} \mapsto h_v^{4}$
\end{tabular}\right\}$ ENCODER

\medskip
\hskip1cm (classifier): $\R^2 \lra \R^4$,  $h_v^{4} \mapsto
\ell_W(h_v^{4})
\, \,$ 
DECODER

\medskip
The encoding space has dimension $2$ and the \textit{encoding function}
$E$ obtained is:
$$
E:\R^{34} \lra \R^2, \quad E=\mathrm{conv3} \circ \mathrm{conv2} \circ \mathrm{conv1}
$$
The decoder consists in a simple linear classifier
$\ell_W: \R^2 \lra \R^4$, where $W$ a $4 \times 2$ matrix of weights,
the image being in $\R^4$, since we have 4 labels.


\medskip
In the training, according to the method elucidated in Sec.
\ref{training-sec}, we choose randomly 4 nodes, one for each class, with their labels
and we hide the labels of the remaining 30 nodes, that we use
as validation. We have no test dataset.
In Fig. \ref{fig:4_embedding}, we plot the node
embeddings in $\R^2$, that is the result of the 34 node coordinates
after the encoding. As one can readily see from Fig \ref{fig:4_embedding},
our dataset, i.e. the nodes, appears more and more linearly
separable, as we proceed in the training. 
The final accuracy is exceeding 70\% on the 4 classes, a remarkable
result considering that our training dataset consists of 4 nodes
with their labels. This is actually a feature of Graph Neural Networks algorithms: typically we need far
less labelled samples in the training datasets, in comparison with the Deep Learning algorithms. 
This is due to the fact that the topology of the dataset already encodes
the key information for an effective training and that the features of the validation set are a priori visible to the training nodes.
This important characteristic makes this algorithm
particularly suitable for biological datasets, though it leaves
to the researcher the difficult task to translate the information
coming from the dataset into
a meaningful graph.  

\subsection{Graph Attention Networks}
\label{gat-sec}

We conclude our short treatment with the 
\emph{Graph Attention Networks} (GATs) \cite{gat}. The diffusivity nature
of the message passing mechanism makes all edges connected to the
same node equally important, while it is clear that in some concrete
problems, some links may have more importance than others and we might have links that exists even if the vertices have very different labels (heterophily). The convolution described in the previous section may not perform well in this setting. A first fix to this problem is to consider weighted adjacency matrices, whose weights can be learnt together with the other weights of the network during training.

Another approach is given by the \emph{graph attentional layers}, that we are going to describe, that have the purpose to assign a weight to the edges connected to
a given node with a particular "attention mechanism" to highlight (or suppress) their importance in diffusing
the features from the node. There are also more advanced algorithms to handle hetherophilic 
datasets\footnote{Hetherophilic dataset are those whose linked nodes do not share similar features.},
like Sheaf Neural Networks (see \cite{bodnar2022}), but we shall not discuss them here. We describe in detail a single attention convolution layer, introduced in the seminal \cite{gat}.
Let $W$ be a weight matrix and
${h}_v, {h}_u \in \mathbb{R}^d$ the (vector valued)
features of nodes $v$ and $u$.
To start with, we define an {\it attention mechanism} as follows:
$$
\begin{array}{ccc}
    a:  \mathbb{R}^{d'} \times \mathbb{R}^{d'} &\longrightarrow & \mathbb{R} \\
          W {h}_v , W {h}_u &\longrightarrow & e_{vu} 
\end{array}
$$

The coefficients $e_{vu}$ are then normalized using the softmax function:
\begin{equation*}
    \alpha _{vu} = 
    \frac{\exp (e_{vu})}{\sum _{w \in \mathcal{N}(v)} \exp (e_{vw})} .
\end{equation*}
obtaining the {\it attention coefficients} $\alpha_{vu}$.
The message passing mechanism is then suitably modified to become:
\begin{equation*}
     {h}_v ' = \sigma \left( \sum_{u \in \mathcal{N}(v)} \alpha_{vu} W  {h}_u \right) .
\end{equation*} 

In the paper \cite{gat}, Velickovic \textit{et al.} choose the following attention mechanism.
First, we concatenate the vectors $W  {h}_v$ and $W  {h}_u$ of dimension $d'$
and then we perform a scalar multiplication by a weight vector $ {a} \in \mathbb{R}^{2 d'}$. Then, as commonly happens, we compose with a LeakyRELU non-linearity 
(a modified version of the RELU)\footnote{The choice of the activation function here is not fundamental, we mention the Leaky Relu only to stick to Velickovic 
\textit{et al.} treatment.}. The resulting attention coefficients are then
\begin{equation*}
    \alpha _{vu} = \frac{\exp \left( \sigma \left(  {a}^{\top} [ W  {h}_v || W  {h}_u ] \right) \right)}
    {\sum _{w \in \mathcal{N}(v)} \exp \left( \sigma \left(  {a}^{\top} [W  {h}_v || W  {}_w ] \right) \right)}
\end{equation*}
where $||$ is the concatenation operation and $\sigma$ is
the LeakyRELU activation function. After having obtained the node embeddings $ h_v '$, Velickovic \textit{et al.} introduce an optional \emph{multi-head attention} that consists into concatenating $K$ identical copies of the embeddings $ {h}_v '$ that will become the node embeddings passed to the following layer. The final message passing becomes then 
\begin{equation*}
     {h}_v ' = \parallel_{i=0}^K\sigma \left( \sum_{u \in \mathcal{N}(v)} \alpha_{vu} W  {h}_u \right) .
\end{equation*}
$K$ is also called the number of \emph{heads}. A layer of this type in the encoder of a GNN is called Graph Attention Layer and a GNN whose encoder consists of a sequence of such layers is called Graph Attention Network (GAT).\\To see a concrete example of a classification problem successfully solved by GAT, we consider the example of the Cora dataset, which is studied also by 
Velickovic \textit{et al.} in \cite{gat}. 
The Cora dataset is an undirected graph consisting of
\begin{itemize}
\item 2708 nodes, each representing a computer science paper,
\item 5209 edges, representing paper citations.
\end{itemize}
Each node $v$ is given a feature $h_v$ consisting of a 1433-dimensional vector corresponding to a bag-of-words representation of the title of the document and a label assigning each node to one of 7 distinguished classes (Neural Networks, Case Based, Reinforcement Learning, Probabilistic Methods, Genetic Algorithms, Rule Learning, and
Theory). Velickovic \textit{et al.} build the following architecture:
$$
\begin{array}{rl}
\mathrm{(conv1)}: 
\mathrm{ELU}(\mathrm{GATConv}(1433, 8)) & \mathrm{heads}=8  \\
\mathrm{(conv2)}: \sigma(\mathrm{GATConv}(64, 7)) & \mathrm{heads}=1
\end{array}
$$ 
where we have denoted as $\mathrm{ELU}$ the Exponential Linear Unit activation function (a slight modification of the RELU, see \cite{Elu}) and as $\sigma$ the softmax activation used for the final classification. Using a training set consisting of 20 nodes per class, a validation set of 500 nodes (to tune the hyperparameters) and a test set of 1000 nodes, the above network achieves an accuracy of the 83\%.

{\bf Acknowledgements}. R.F. wished to thank Prof. S. Soatto, Dr. A. Achille
and Prof. P. Chaudhari for the patient explanations of the functioning
of Deep Learning. R.F. also wishes to thank Prof. M. Bronstein,
Dr. C. Bodnar for helpful discussions on the geometric deep learning on graphs. F.Z. thanks Dr. A. Simonetti for helpful discussion on both the theory and the practice of Geometric Deep Learning Algorithms.

\appendix

\section{Fisher matrix and Information Geometry} \label{infogeo}
In this appendix we collect some known facts about the Fisher information
matrix and its importance in
Deep Learning.
Our purpose is to establish 
a dictionary between Information Geometry and
(Geometric) Deep Learning questions, to help with the research in both.
No prior knowledge beyond elementary probability theory
is required for the appendix.
We shall focus on {\sl discrete probability} distributions only, 
since they are the ones interesting for our machine learning applications.
For more details we send the interested reader to
\cite{amari}, \cite{martens} and refs. therein.

\medskip
Let $p(x,w)=(p_0(x,w), \dots, p_{C-1}(x,w))$
be a discrete probability distribution, representing
the probability that a datum $x$ is assigned a class among $0,\dots, C-1$.
In machine learning $p(x,w)$ is typically an {\sl empirical probability} and depends on
parameters $w \in \R^p$. As we described previously, during
training $p(x,w)$ changes and the very purpose of the training
is to obtain a $p(x,w)$ that gives a good approximation of the true
probability distributions $q(x)$ on the training set, and of course, also
performs well on the validation set.

\begin{definition}
  Let $p(x,w)$ be a discrete probability distribution, we define
  \textit{information loss} of $p(x,w)$ as 
\beq\label{amariloss}
  I(x,w)=-\log(p(x,w))
\eeq
\end{definition}

In \cite{amari}, Amari refers to $I$ as the {\sl loss function},
however, given the importance of the loss function in Deep Learning, we prefer to
call $I(x,w)$ {\sl information loss}. Notice
that  $I(x,w)$ is a probability distribution.

\medskip
The information loss is very important in Deep Learning: its expected value
with respect to the 
true probability distribution $q(x)$ as defined in
(\ref{qmass}) is the cross entropy loss, one the most used
loss functions in Deep Learning.

\begin{definition}
  Let $p(x,w)$ be a discrete empirical probability distribution,
  $q(x)$ the corresponding true distribution.
  We define \textit{loss function} the expected value of the
  information loss with respect to $q(x)$:
  $$
  L(x,w)= - \bE_q[\log p(x,w)] =-\sum_{i=1}^C q_i(x)
\log p_i(x,w)
  $$
\end{definition}

Given two probability distributions $p$ and $q$, the Kullback Leibler
divergence intuitively measures how much they differ and it is
defined as:
$$
\KL(q||p):=\sum_i q_i \log \frac{q_i}{p_i}
$$

As we have seen in Sec. \ref{loss-subsec} we have that:
$$
\KL(q(x)||p(x,w))=L(x,w)+H(q(x))=L(x,w) 
$$
since $H(q(x))=0$ for a mass probability density.

\medskip
Now we turn to the most important definition in Information
Geometry: the Fisher information matrix.

\begin{definition}
  Let $p(x,w)$ be a discrete empirical probability distribution,
  $q(x)$ the true distribution.
  We define $F$ the
  \textit{Fisher information matrix} or \textit{Fisher matrix}
  for short, as
  $$
  F_{ij}(x,w)=\bE_p[\partial_{w_i}\log(p(x,w)\partial_{w_j}\log(p(x,w)]
  $$
\end{definition}

One way to express coincisely the Fisher matrix is
the following:
$$
F(x,w)=\bE_p[\nabla \log(p(x,w) (\nabla \log(p(x,w))^t]
$$
where we think $\nabla \log(p(x,w) $ as a column vector and $T$
denotes the transpose.
Notice that $F(x,w)$ does not contain any information regarding
the true distribution $q(x)$.

\begin{remark}
  Some authors prefer the Fisher matrix to depend 
  on the parameters only, hence they take a sum over the data:
$$
F(w)=\sum_x F(x,w)
$$
We shall not take this point of view here, adhering to
the more standard treatment as in \cite{amari}.
\end{remark}

\begin{observation}
  Notice that $F(x,w)$ is symmetric, since it is a finite sum
  of symmetric matrices. Notice also that it is positive semidefinite.
  In fact
  \begin{align}
\begin{split}
u^t F(x,w) u &=
\bE_{p}\left[u^t \nabla_w\log p(x,w) (\nabla_w \log p(x,w))^t u\right] =\\
&=\bE_{p}\left[\langle \nabla_{w}\log p(x, w), u \rangle^{2} \right] \geq 0.
\end{split}
  \end{align}
  where $\langle,\rangle$ denotes the scalar product in $\R^p$.
\end{observation}

The next proposition shows that the rank of the Fisher matrix
is bound by the number of classes $C$. This has a great impact
on the Deep Learning applications: in fact, while the Fisher
matrix is quite a large matrix, $p\times p$, where $p$ is
typically in the order of millions, the number of classes is
usually very small. Hence the Fisher matrix has a very low
rank compared with its dimension; e.g. in MNIST, the rank
of the Fisher is no larger than 9, while its dimension is of
the order typically above $10^4$.

\begin{proposition} \label{rank}
  Let the notation be as above. Then:
  $$
  \mathrm{rk} F(x,w) < C
  $$
\end{proposition}

\begin{proof}
We first show that $\ker F(x, w) \subseteq
(\Span_{i=1, \ldots, C}\{\nabla_w \log p_i( x,w)\})^\perp$:
\begin{align}
\begin{split}
u \in \ker F(x, w) &\Rightarrow u^{T}F(x, w)u = 0
\Rightarrow 
\bE_{p}\left[\langle \nabla_{w}\log p( x, w), u \rangle^{2} \right] = 0 \\
&\Rightarrow 
\langle \nabla_{w} \log p_i( x, w), u \rangle = 0 \quad \forall\ i=1,\ldots, C.
\end{split}
\end{align}
On the other hand, if $u \in (\Span_{i=1, \ldots, C}\{
\nabla_w \log p_i( x,w)\})^\perp$ then $u \in \ker F(x, w)$:
\begin{equation}
F(x, w)u = \bE_{p}\left[ \nabla_{w}\log p_i( x, w)
\langle \nabla_{w}\log p_i( x, w), u \rangle \right] = 0.
\end{equation}
This shows that $\rank\ F(x,w) \leq C$.

The vectors $\nabla_{w}\log p_i( x, w)$ are linearly dependent since
\begin{align}
\begin{split}
\label{eq:expected_value}
& \bE_{p}[\nabla_w I(x,w)]=\sum_{i=1}^C p_i( x,w) \nabla_w \log p_i( x,w)= \\
&= \! \sum_{i=1}^C \nabla_w p_i( x,w)= 
\nabla_w \! \! \left(\sum_{i=1}^C p_i( x,w)\!\right) \! = \! \nabla_w 1 = 0.
\end{split}
\end{align}
Therefore, we deduce $\rank\ F(x,w) < C$.
\end{proof}

We now relate the Fisher matrix to the information loss.

\begin{proposition}
  The Fisher matrix 
  is the covariance matrix of the gradient of the information loss.
\end{proposition}

\begin{proof} The gradient of the information loss is 
$$
\nabla_w I(x,w)=- \frac{\nabla_w p( x,w)}{p( x,w)}
$$
Notice:
$$
\begin{array}{l}
\bE_{p}(\nabla_w I)=
\sum p_i\frac{\nabla_w p_i}{p_i}=
\sum_i\nabla_w p_i=
\nabla_w(\sum_i p_i)=0
\end{array}
$$ 
The covariance
matrix of $\nabla_w I(x,w)$ is (by definition): 
$$
\begin{array}{l}
\mathrm{Cov}(I)=\bE_{p}[(\nabla_w I -\bE_{p}(\nabla_w I))^t
(\nabla_w I -\bE_{p}(\nabla_w I))]= \\ \\
\qquad=\bE_{p}[(\nabla_w I)^t(\nabla_w I)]=F(x,w)
\end{array}
$$
\end{proof}

We conclude our brief treatment of Information Geometry by some observations
regarding the metric on the parameter space.

\begin{observation}
We first observe that:
$$
\begin{array}{rl}
\KL(p( x,w+\delta w)||p( x,w)) &\cong 
 \frac{1}{2}(\delta w)^t F(x,w) (\delta w) 
+\mathcal{O}(||\delta w||^3)
\end{array}
$$
This is a simple exercise based on Taylor expansion of the $\log$ function.

\medskip
Let us interpret this result in the light of Deep Learning,
more specifically, during the dynamics of stochastic gradient descent.
The Kullback-Leibler divergence  $\KL(p( x,w+\delta w)||p( x,w))$ measures
how $p( x,w+\delta w)$ and $p( x,w)$ differ, for a small variation of the parameters $w$,
for example for a step in stochastic gradient descent.
If $F$ is interpreted as a metric as in \cite{amari}
\footnote{
The low rank of $F$ in Deep Learning, 
makes this interpretation problematic.}, 
then $(\delta w)^t F(x,w) (\delta w)$ expresses
the size of the step $\delta w$. 
We notice however that, because of Prop. \ref{rank}, 
this interpretation is not satisfactory. Even if we restrict
ourselves to the subspace of $\R^p$ where the Fisher matrix is non degenerate, we cannot construct a submanifold, due to the non integrability of the distribution of the gradients, that we observe experimentally (see \cite{gf}).
Moreover the dynamics does not even follow a sub-riemannian constraint either, due to the non constant rank of the Fisher.
This very challenging behaviour
will be the object of investigation of a forthcoming paper.
\end{observation}

In the next proposition, we establish a connection between Information
Geometry and Hessian Geometry, since the metric, given by $F$,
can be viewed in terms of the Hessian
of a potential function.

\begin{proposition}
Let the notation be as above. Then
$$
F(x,w)=\bE_{p}[\mathbb{H}(I(x,w))]
\quad \hbox{where} \quad I(x,w)=-\log p(x,w)
$$
\end{proposition}

\begin{proof}
  In fact (write $p=p( x,w)$):
$$
\mathbb{H}[I]=-\mathrm{Jac}\left[\frac{\nabla_w p}{p}\right]=
-\left[\mathbb{H}(p) \cdot p + \nabla_w p \cdot \nabla_w p\right]\frac{1}{p^2}
$$
Take the expected value: 
$$
\bE_{p}[\mathbb{H}[I]]=-\sum_i
p_i \frac{\mathbb{H}(p_i)}{p_i}+\bE_{p}\left[\frac{\nabla_w p}{p} 
\cdot \frac{\nabla_w p}{p}\right]=F
$$
where $\sum_i\mathbb{H}(p_i)=\mathbb{H}(\sum_ip_i)=0$.
\end{proof}

\section{Regression tasks} \label{regr-app}
In this appendix we provide some explanations and references on another very important learning task that can be handled using Deep Learning: regression. 
In the main text we focused on \emph{classification} tasks, such as image recognition. We may be interested in other practical problems as well, for example we may want to predict the price of an house or of a car given some of its features such as dimension, mileage, year of construction/production, etc. In general, given some numerical data we may want to predict a number, or a vector, rather than a (probability of belonging to) a class. These tasks are called \emph{regression tasks} to distinguish them from classification tasks. From an algorithmic point of view, regression tasks are handled and trained as classification tasks, the most important difference being that different loss functions in step 2 of Section \ref{scd-sec} are usually employed. In particular, steps 0, 1, 3, 4, 5 described in Section \ref{scd-sec} can be repeated almost verbatim with the only conceptual difference that we are not trying to predict a class but rather to infer a vector: therefore it is more appropriate to speak of a \emph{regression or prediction function} instead of a score function, etc.\\For example, for step 1 if our regression task consists in predicting $p$ dimensional vectors from $d$ dimensional features, we choose a suitable regression or prediction function $P(x,w):\R^d\times\R^p\rightarrow\R^c$ that we can think of a score function where the codomain has not to be thought as a "vector of probabilities" but rather as the actual prediction or regression of our algorithm.
For step 2, suppose that we have a regression function
$P(x,w):\R^d\times\R^p\rightarrow\R^c$ and that we denote, for any sample $x_i\in\R^d$, $i=1,...,N$, with $y_i\in\R^c$ the vector we would like to predict. The following are then some common loss functions we can use for training:
\begin{itemize}
\item \emph{Mean Squared Error} (MSE) or $L_2$-norm: $$L(w):=\frac{1}{N}\sum_{i=1}^N||P(x_i,w)-y_i||_2^2$$
\item \emph{Root Mean Squared Error} (RMSE): $$L(w):=\sqrt{\frac{1}{N}\sum_{i=1}^N||P(x_i,w)-y_i||_2^2}$$
\item \emph{Mean Absolute Error} (MAE) or $L_1$-norm: $$L(w):=\frac{1}{N}\sum_{i=1}^N||P(x_i,w)-y_i||_1$$
\end{itemize}
The MSE is one of the most important loss functions when it comes to regression tasks and in the past also classification tasks were sometimes treated as regression tasks and the MSE loss was used in these cases as well. 
\begin{remark}
Performing a regression task using the MSE loss function to train a single layer MLP model is equivalent to solve an ordinary linear regression (see \cite{Hastie2005}) problem using a gradient descent algorithm, as a single layer MLP is simply an affine map.
\end{remark}
A priori, loss functions whose first-order derivatives do not exist might be problematic for stochastic descent algorithm (that is why the MSE is sometimes preferred over the MAE, for example). To solve this issue in some practical cases when a non differentiable loss function like the MAE would be appropriate for the regression problem at hand but problematic from the training viewpoint, the solution is often to use more regular functions that are similar to the desired non-differentiable loss function whose behaviour is desired. For example the  "Huber" or the "log(cosh)" loss functions could used in place of the MAE, see \cite{SalehLogCosh}.

\section{Multi-layer perceptrons and Convolutional Neural Networks}\label{mlp-app} 
In this appendix we give a self-contained explaination of multi-layer perceptrons and convolutional neural networks, complementing the one we give in the main body of the paper. Given an affine map $F:\R^n\rightarrow\R^m$, $F(x)=Ax+b,\quad A\in\M_{m,n}(\R),b\in\R^m$, we will call $b$ the bias of $F$. By viewing a matrix $A\in\M_{m,n}(\R)$ as a vector of $\R^{mn}$ we can define an \emph{affine (or, abusing notation linear) layer} as a couple $(F,w)$ where $F$ is an affine map and $w$ is the vector of weights obtained concatenating $A$ and the bias vector. One dimensional convolutions are the discrete
  analogues of convolutions in mathematical analysis. Let us consider a vector $x \in \R^d$. As already did in the main text, we can define a simple one dimensional convolution as follows:
  $$
  \mathrm{conv1d}:\R^d \lra \R^{d-r},\quad
  (\mathrm{conv1d}(x))_i=\sum_{j=1}^{r} K_{j}x_{i+j}
  $$
  where $K$ is called a \textit{filter} and is a vector of weights
  in $\R^r$, $r<d$. This should serve as the basic example to keep in mind and the analogy with continuous convolutions is clear in this case.

To generalize to the case of 2-dimensional convolutions and to obtain more general 1-dimensional convolutions, it is convenient to define one dimensional convolutions to be the following slightly more general functions:
\begin{definition}
    Let be $K\in\R^r$, and consider index functions $a(i,j), \alpha(i,j):[l]\times[d]\rightarrow\N$, where we have denoted as $[n]$ the subset $\lbrace 1,\cdots n\rbrace\subset\N$ and $a(i,-):[d]\rightarrow\N$ is assumed to be injective for all $i\in [l]$. We define \emph{one dimensional convolution operators} as functions of the following form
    $$
  \mathrm{conv1d}:\R^d \lra \R^{l},\quad
  \mathrm{conv1d}(x)_i=\sum_{j=1}^{d} K_{\alpha(i,j)}x_{a(i,j)} +b_i
   $$
   where $\R^d\ni x=(x_i)_{i=1}^d$, we set $x_k=0$, $K_h=0$ if $k\notin [d]$, $h\notin [r]$ and $b_i\in \R$ is a \emph{bias} term. We will call the vector $w\in\R^{r+l}$ resulting from the concatenation of $K$ and all the biases the \emph{vector of weights} associated to $\mathrm{conv1d}$. A one dimensional \emph{convolutional layer} is a couple $(\mathrm{conv1d},w)$ where $\mathrm{conv1d}$ is a one dimensional convolution and $w$ is the vector of weights associated to it.
\end{definition}
\begin{remark}
Under the assumptions and the notations of the previous definition if we set $a(i,j):=i+j$, $\alpha(i,j)=j$, $l=d-r$ and $b_i=0$ for all $i=1,...,l$ we get the simple convolution we defined before. In addition, the index functions $a(i,-)$ and $\alpha(i,j)$ are usually chosen to be non-decreasing, thus resembling the discrete analogue of the convolution usually employed in analysis.
\end{remark}
Besides the simple convolution we defined at the beginning of this paragraph, there are some standard choices of the functions $a(i,j)$, $\alpha(i,j)$ that are controlled by parameters known as \emph{stride, padding}, etc. We will not introduce them here as they become useful only when constructing particular models: we refer to \cite{lecun2015} for a discussion.\\
Many convolutions used in practice, in addition to having increasing index functions, have $l,r\leq d$ and are said to have \emph{one output channel}, or are of the form $\mathrm{conv1d}^{\times e}:\R^d\rightarrow\R^{l\times e}$ where $\mathrm{conv1d}$ is a one output channel, one dimensional convolution. The number $r$ in these cases is called the \emph{kernel size}.
We can define also the so called "pooling functions".
\begin{definition}
Consider an index function $a(i,j):[l]\times[d]\rightarrow\N$, as in the previous definition. For any $i\leq l$ and any vector $x\in\R^d$, we denote as $x_{a(i,-)}\in\R^R$ the vector $(x_{a(i,j)})_{j=1}^d$. We say that a function $P:\R^d\rightarrow\R^l$ is a \emph{pooling layer} or operator if $P(x)_i=\varphi(x_{a(i,-)})$ for a fixed \emph{pooling function} $\varphi$. If $\varphi$ is the arithmetic mean or the maximum function we will call the layer $P$ \emph{mean pooling} or \emph{max pooling} layer respectively. A pooling layer does not have an associated set of weights.
\end{definition}
\begin{remark}
Pooling layers are not meant to be trained during the training of a model, therefore they do not have an associated vector of weights.
\end{remark} 

\medskip
Two dimensional convolutions are conceptually the analogue of one dimensional convolutions in the context of matrices. Indeed, some data like images can be more naturally thought as grids (matrices) of numbers rather than vectors. Even if two dimensional convolutions are a particular case of one dimensional convolutions, because of their importance and their widespread use it is important to discuss them on their own. 

\begin{definition}
Denote as $\M_{d,q}(\R)$ the space of $d\times q$ real valued matrices. Let be $K_{hk}\in\M_{d,q}(\R)$, and let be $a(i,j), \alpha(i,j):[r]\times[n]\rightarrow\N$, $b(i,j), \beta(i,j):[s]\times[m]\rightarrow\N$ index functions where $a(i,-), b(j,-)$ are assumed to be injective for all $i, j$. We define two dimensional convolution operators as functions of the following form
    $$
  \mathrm{conv2d}:\M_{n,m}(\R) \lra \M_{r,s}(\R),\quad
  (\mathrm{conv2d}(A))_{ij}=\sum_{h=1}^{n}\sum_{k=1}^{m} K_{\alpha(i,h)\beta(j,k)}A_{a(i,h)b(j,k)} +b_{ij}
   $$
   where $A_{ij} \in \M_{n,m}(\R)$, as in the case of one dimensional convolutions we set $A_{ij}=0$, $K_{hk}=0$ if either $i>n$, $j>m$, $h>d$ or $k>q$ and $b_{ij}\in R$ is a \emph{bias} term.
\end{definition}
\begin{remark}
Using the canonical isomorphism $\M_{n,m}(\R)\cong\R^{n\times m}$ (notice that it is $\mathbb{Z}$-linear) 
we can see that there is a bijection between the set of two dimensional convolutions and the one of one dimensional convolutions. As a consequence, we define a 2-dimensional \emph{convolutional layer} as a couple $(C,w)$ where C is a 2-dimensional convolution and $w$ is its associated vector of weights, obtained as the vector of weights of the one dimensional convolution associated to $C$. Moreover, as in the case of one dimensional convolutions, the index functions $a(i,-), b(j,-),\alpha(i,-), \beta(j,-)$ are usually assumed to be non-decreasing.
\end{remark}
As in the case of one dimensional convolutions, the form of the index functions are usually standard and controlled by parameters widely used by the practitioners such as stride, padding, etc., and the numbers $d,q$ are called the \emph{kernel sizes}. In addition, in most cases $p=q$ (i.e. the filter or kernel matrix is a square matrix).\\

\begin{figure}[h!]
    \centering
    \includegraphics[width=0.75\textwidth]{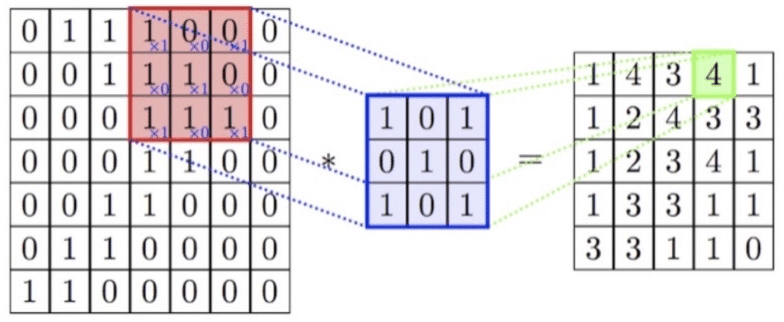}
      \caption{A 2d convolution as depicted in \cite{tobaccocnn}}
    \label{CNNconv}
\end{figure}

As we mentioned at the beginning of Section \ref{score-sec} we are not only interested on convolutions acting on $\M_{n,m}(\R)\cong\R^n\times\R^m$ but also on convolutions acting on $\R^p\times\R^n\times\R^m$ representing an image having $p$ channels or a $p\times n\times m$ voxel. We can then define 3-dimensional convolutions along the lines of what we did for 2-dimensional ones and, more generally we can define $n$-dimensional convolutions and convolutional layers. As these are straightforward generalizations and as all these cases reduce to the case of one dimensional convolutions, we will not spell out the definitions. We define $n$-dimensional \emph{pooling layers} analogously.\\
We shall now define two very important types of neural networks.
\begin{definition}
We say that $F$ ($(F,w)$) is a \emph{convolution} (convolution layer) if it is a convolution operator (layer). We say that $\sigma:\R\rightarrow\R$ is an \emph{activation function} if it is either a bounded, non-constant continuous function or if it is the RELU function. Given an activation function $\sigma$, for all $n\geq 1$ we can view it as a map $\R^n\rightarrow\R^n$ acting componentwise (and in this case we denote $\sigma^{\times n}$ as $\sigma$ abusing terminology).
\end{definition}

\begin{definition}
Consider a finite sequence of positive natural numbers $N_0, ...,N_L$\footnote{That, under the terminology of Section \ref{scd-sec}, are among the hyperparameters}, $L\in\N_{>0}$ and fix an activation function $\sigma$. We say that a function $F:\R^{N_0}\rightarrow\R^{N_L}$ is called
\begin{itemize}
    \item A $\sigma$-\emph{multilayer perceptron} (MLP) if it is of the form $$F=G\circ H_{L-1}\circ\cdots \circ H_{1}$$
    where $G:\R^{N_{L-1}}\rightarrow\R^{N_L}$ is an affine map and for all $i=1,...,L-1$, $H_i=\sigma\circ A_i$ where $A_i:\R^{N_{i-1}}\rightarrow\R^{N_i}$ is an affine map.
    \item A $\sigma$-\emph{convolutional neural network} (CNN) if it is of the form $$F=G\circ H_{L-1}\circ\cdots \circ H_{1}$$ where $G:\R^{N_{L-1}}\rightarrow\R^{N_L}$ is an affine map and, for all $i=1,...,L-1$, $H_i$ is either a pooling function or it is of the form $H_i=\sigma\circ A_i$ where $A_i:\R^{N_{i-1}}\rightarrow\R^{N_i}$ is affine or a convolution.
\end{itemize}
The number $L$ is usually called the \emph{number of layers} and $\sigma$ is assumed to be applied componentwise. The convolutions and the affine maps appearing in a MLP or in a CNN are considered as layers. We denote as $\cN^\sigma(\R^n,\R^m)$ the set of $\sigma$-MLPs having domain $\R^n$ and codomain $\R^m$. Finally, given a MLP or a CNN $F$, we define the vector of its \emph{weights}, $w_F$, to be the concatenation of all the vectors of weights of its layers. When we see a MLP or a CNN $F$ as a network, we usually think of it as a couple $(F(w),w)$, and in the terminology of Section \ref{dl-sec}, somewhat abusing the terminology, that $w$ is the model.
\end{definition}
There exists many more types of CNNs, therefore our definition here is not meant to be the most general possible. However, all of them use at least one convolution operator as the ones the we have defined before. MLPs are usually referred as "fully-connected" feedforward neural networks. Indeed, both MLPs and CNNs fall under the broader category of \emph{feeedforward neural networks} which are, roughly speaking, networks that have an underlying structure of directed graph. As there is not a general straightforward and concise definition of these networks, in this work we will only consider MLPs and CNNs having the structure above and refer the reader to the literature, see \cite{lecun2015} for more general definitions or Section2 in [Zhu22] for a recent and very clean definition from a graph theoretical viewpoint.\\

\begin{figure}[h!]
    \centering
    \includegraphics[width=0.75\textwidth]{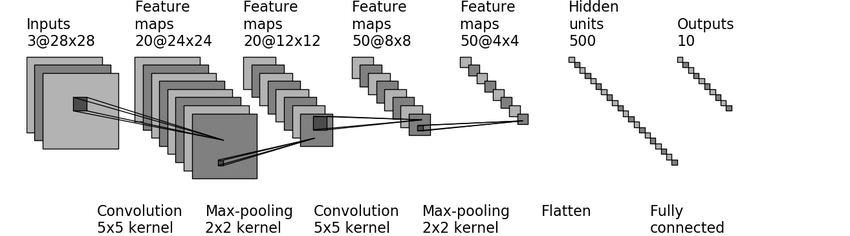}
      \caption{A very influential Convolutional Neural Network: the famous LeNet introduced in \cite{LeNet} (the image depicts LeNet-5 as in \cite{lenetimg}).}
    \label{CNNconv}
\end{figure}
Suppose we are given a network $(F(w),w)$ which can be either a MLP or a CNN. If $w\in\R^p$ and $F(w):\R^n\rightarrow\R^C$ 
we can define the score function associated to $(F(w),w)$ as $$s:\R^d\times\R^p\rightarrow\R^C,\quad s(x,w):=F(w)(x)$$This is the score we use for training as explained in the main text.

\section{Universal Approximation Theorem}\label{Univ:approx}
We conclude this appendix mentioning a so called \emph{universal approximation theorem} for MLPs, an important theoretical result showing that the architectures we introduced have a great explanatory power. For more details, see
\cite{Hornik1991} Theorems 1 and 2. 

\begin{theorem}
For any given activation function $\sigma$, the set of $\sigma$-MLPs $\cN^\sigma(\R^n,\R)$ is dense in the set of continuous real valued functions $C(\R^n)$ for the topology of uniform convergence on compact sets.
\end{theorem}
\begin{proof}
This is proved as Theorem 2 in \cite{Hornik1991}. We only remark that the theorem holds in the case $\sigma$ is the RELU, see the remark at 
page 253 in \cite{Hornik1991}, as for a given $\R\ni a\neq 0$ the "spike function" $$r(x)=\mathrm{RELU}(x-a)-2\mathrm{RELU}(x)+\mathrm{RELU}(x+a)$$ is bounded, continuous and non-constant
\end{proof}
This theorem is conceptually very important as it states that an arbitrary continuous function can be in principle approximated by a suitable MLP. As a consequence, trying to approximate unknown functions with Neural Networks may seem less hopeless
than what one might think at a first glance. In real world applications, though it is somewhat unpractical to solve any task with an MLP (for many reasons such as lack of data), and for this reason other architectures such as CNNs, GNNs etc. have been designed and deployed.
\\
There exist more general approximation theorems concerning more general MLPs (e.g. for the class $\cN^\sigma(\R^n,\R^m)$), or for other neural networks architectures such as CNNs and GNNs. For these results and for a more comprehensive discussion of the topic the reader is referred to 
\cite{univ-thm}, \cite{lecun2015} and \cite{noneuclapprox}.

\bibliographystyle{plain}
\bibliography{dl-notes2}
 
\end{document}